\documentclass{article}

\usepackage[final,nonatbib]{neurips_2023}

\usepackage[numbers]{natbib}

\usepackage{times}
\usepackage{url}            %
\usepackage{booktabs}       %
\usepackage{amsfonts}       %
\usepackage{nicefrac}       %
\usepackage{microtype}      %
\usepackage{mkolar_definitions}
\usepackage{xspace}
\usepackage{multirow}
\usepackage{amsmath}
\usepackage{algorithm}
\usepackage{algorithmic}
\usepackage{color}
\usepackage{color, colortbl}
\usepackage{enumitem}
\usepackage{comment}
\usepackage{bm}
\usepackage{subcaption}
\usepackage{wrapfig}

\usepackage{thmtools}
\newtheorem{myexp}{Example}

\usepackage{url}
\usepackage{csquotes}
\usepackage{tikz}

\usepackage[colorlinks=true, linkcolor=blue, citecolor=blue]{hyperref}

\usepackage{centernot}

\newcount\Comments  %
\definecolor{darkgreen}{rgb}{0,0.5,0}
\definecolor{darkred}{rgb}{0.7,0,0}
\definecolor{teal}{rgb}{0.3,0.8,0.8}
\definecolor{orange}{rgb}{1.0,0.5,0.0}
\definecolor{purple}{rgb}{0.8,0.0,0.8}
\newcommand{\kibitz}[2]{\ifnum\Comments=1{\textcolor{#1}{\textsf{\footnotesize #2}}}\fi}
\usetikzlibrary{positioning}

\definecolor{Gray}{gray}{0.9}

\newcommand{\Prr}{\mathrm{Pr}}
\newcommand{\VM}{\mathrm{VM}}

\newcommand{\rI}{\mathrm{I}}

\usepackage[utf8]{inputenc} %
\usepackage{hyperref}       %
\usepackage{url}            %
\usepackage{booktabs}       %
\usepackage{amsfonts}       %
\usepackage{nicefrac}       %
\usepackage{microtype}      %
\usepackage{xcolor}         %

\title{Future-Dependent Value-Based \\ Off-Policy Evaluation in POMDPs}

\author{%
 \vspace{-1em} Masatoshi Uehara \thanks{Co-first author. This work was done at Cornell University}  \\
\vspace{-1em} Genentech  \\ 
  \texttt{uehara.masatoshi@gene.com} \\
  \And 
  \vspace{-1em} Haruka Kiyohara \thanks{Co-first author. This work was done at Tokyo Institute of Technology }\\ 
   \vspace{-1em} Cornell University  \\ 
\vspace{-1em}\texttt{hk844@cornell.edu} \\   
\And 
\vspace{-1em}  Andrew Bennett  \\ 
\vspace{-1em} Morgan Stanley \thanks{This work was done at Cornell University}\\
\vspace{-1em} \texttt{Andrew.Bennett@morganstanley.com} \\ 
 \And 
\vspace{-1em} Victor Chernozhukov \\
\vspace{-1em} MIT \\ 
\vspace{-1em} \texttt{vchern@mit.edu} \\ 
\And 
\vspace{-1em} Nan Jiang \\ 
\vspace{-1em} UIUC \\ 
\vspace{-1em} \texttt{nanjiang@illinois.edu}  \\ 
\And 
  \vspace{-1em} Nathan Kallus \\ 
   \vspace{-1em} Cornell University \\ 
    \vspace{-1em}  \texttt{kallus@cornell.edu} \\ 
    \And
\vspace{-1em}   Chengchun Shi \\ 
\vspace{-1em}  LSE \\ 
\vspace{-1em} \texttt{c.shi7@lse.ac.uk}  \\
\And 
 \vspace{-1em}   Wen Sun  \\
  \vspace{-1em} Cornell University \\
   \vspace{-1em} \texttt{ws455@cornell.edu} \\
}
\usepackage[compact]{titlesec}

\begin{document}

\maketitle

\begin{abstract}
 We study off-policy evaluation (OPE) for partially observable MDPs (POMDPs) with general function approximation. Existing methods such as sequential importance sampling estimators suffer from the curse of horizon in POMDPs.  %
 To circumvent this problem, we develop a novel model-free OPE method by introducing future-dependent value functions that take future proxies as inputs %
 and perform a similar role to that of classical value functions in fully-observable MDPs. We derive a new off-policy Bellman equation for future-dependent value functions as conditional moment equations that use history proxies as instrumental variables. We further propose a minimax learning method to learn future-dependent value functions using the new Bellman equation. We obtain the PAC result, which implies our OPE estimator is close to the true policy value under Bellman completeness, as long as futures and histories contain sufficient information about latent states. Our code is available at \href{https://github.com/aiueola/neurips2023-future-dependent-ope}{https://github.com/aiueola/neurips2023-future-dependent-ope}. 

\end{abstract}

\vspace{-0.1cm}
\section{Introduction} \label{sec:intro}
\vspace{-0.1cm}

Reinforcement learning (RL) has demonstrated success when it is possible to interact with the environment and collect data in an adaptive manner. However, for domains such as healthcare, education, robotics, and social sciences, online learning can be risky and unethical \citep[e.g.,][]{klasnja2015microrandomized,kosorok2019precision,tsiatis2019dynamic,singla2021reinforcement}. To address this issue, a variety of offline RL methods have recently been developed for policy learning and evaluation using historical data in these domains \citep[see][for an overview]{levine2020offline}. In this paper, we focus on the off-policy evaluation (OPE) problem, which concerns estimating the value of a new policy (called the evaluation policy) using offline log data that was generated under a different policy (called the behavior policy) \citep{thomas2016data,jiang2016doubly,tsiatis2019dynamic,chen2022well}. OPE is especially useful in the aforementioned high-stakes domains.

We focus on off-policy evaluation (OPE) in partially observable Markov decision processes (POMDPs). Partial observability is a common phenomenon in practical applications \citep{cassandra1998survey,kaelbling1998planning,xu2020latent}, and it poses a serious challenge for sample-efficient learning. Most existing OPE methods are developed for MDPs and their statistical learning guarantees rely crucially on the Markov assumption. %
To extend these methods to POMDPs, one may 
use the entire history of observations as a \emph{state} to satisfy Markovanity. This allows us to apply sequential importance sampling \citep[SIS,][]{precup2000eligibility,jiang2019entropy,hu2021off} or its variant  \citep[e.g., sequential doubly robust (SDR) methods][]{rotnitzky1998semiparametric,murphy2001marginal,jiang2016doubly,thomas2016data,farajtabar2018more,bibaut2019more} to the transformed data for valid policy evaluation. %
An alternative approach is to employ %
value-based methods by constructing a history-dependent Q-function that takes the entire history as input. These value functions can be estimated via fitted-Q evaluation \citep[FQE,][]{ernst2005tree,munos2008finite,le2019batch} or minimax methods \citep{antos2008learning,chen2019information,feng2019kernel,nachum2019dualdice,uehara2020minimax,uehara2021finite,zanette2022bellman}. However, all aforementioned estimators suffer from the curse of horizon, as their estimation errors grow exponentially with respect to the (effective) horizon and become prohibitively large in long-horizon settings \citep{Liu2018,kallus2020double,kallus2022efficiently}.

The goal of this paper is to devise practical and efficient %
OPE methods %
in large partially observable %
environments, while breaking the curse of horizon. 
As a first step, we restrict our attention to the evaluation of short memory-based policies that take several recent observations instead of the entire history as inputs. Short memory-based policies are widely used in practice, since maintaining the whole previous observation leads to a well-known computational challenge dubbed as the curse of history \citep{pineau2006anytime,golowich2022planning}.  For example, in the deep Q-network (DQN) algorithm, four game frames are stacked together to train the optimal policy in Atari \citep{mnih2013playing}, while Open AI Five uses a window of length 16 \citep{berner2019dota}. Even when considering the evaluation of policies with short-memory or memory-less characteristics (i.e., policies that depend solely on current observations),   na\"ive methods such as SIS, SDR, and FQE still suffer from the curse of horizon\footnote{Notice that even when restricting to short-memory policies, it is necessary to use the entire history instead of the short memory as the state to meet the Markov property.}.%

Our proposal takes a model-free perspective and introduces a new concept called ``future-dependent value functions''. %
The proposed value function does not involve latent states which are unobservable in POMDPs. Nor does it relies on the entire data history. It performs the role of standard value functions by incorporating future observations as proxies for latent states. 
We demonstrate that these future-dependent value functions satisfy a new off-policy Bellman equation for POMDPs, which involves past observations %
to approximate latent states. To estimate these future-dependent value functions from offline data using the new Bellman equation, we propose minimax learning methods that accommodate various function approximation, including deep neural networks, RKHS, and linear models. Notably, our method extends the classical LSTD \citep{lagoudakis2003least} developed in MDPs to the POMDP setting when using linear models. 
To summarize, our contributions are as follows: 
\begin{itemize}[leftmargin=*]
    \item We propose a novel model-free method for OPE that leverages future-dependent value functions. Our key idea is to utilize future and historical observations as proxies for latent states. 
    
    \item We derive a new Bellman equation for learning future-dependent value functions. %
    The proposed estimator can accommodate various types of %
    function approximations.
    \item  We provide PAC guarantees to demonstrate that our method can address the curse of horizon, and conduct numerical experiments to showcase its superiority over existing methods.
\end{itemize}

Note that similar concepts of future-dependent value functions have recently been introduced in the context of OPE in confounded POMDPs \citep{shi2021minimax} and online RL in POMDPs \citep{uehara2022provably}. However, we focus on OPE in \emph{non-confounded} POMDPs. Further detailed discussions %
can be found in Section~\ref{subsec:related}.

\vspace{-0.1cm}
\subsection{Related Literature} \label{subsec:related}
\vspace{-0.1cm}

\vspace{-0.15cm}
\textbf{OPE in confounded POMDPs.} \quad OPE with unmeasured confounders has been actively studied  \citep{zhang2016markov,namkoong2020off,tennenholtz2020off,wang2020provably,kallus2020confounding,liao2021instrumental,nair2021spectral,guo2022provably,shi2022off,xu2022instrumental,lu2022pessimism,chen2023unified,bruns2023robust}. 
Among them, \cite{tennenholtz2020off} adopted the POMDP model to formulate the OPE problem in the presence of unmeasured confounders. %
They borrowed ideas from the causal inference literature on double negative control adjustment \citep{miao2018identifying,miao2018confounding,cui2020semiparametric,kallus2021causal,shi2020multiply} to derive consistent value estimators in tabular settings. Later,  \cite{bennett2021off} and \cite{shi2021minimax} extend their proposal by incorporating confounding bridge functions, thereby enabling general function approximation. (see the difference between these bridge functions and the proposed future-dependent value function in \pref{sec:difference}). However, these methods do not apply to our unconfounded POMDP setting. This is because these methods require the behavior policy to \textit{only} depend on the latent state to ensure certain conditional independence assumptions. These assumptions are violated in our setting, where the behavior policy may depend on the observation -- a common scenario in practical applications. In addition, we show that in the unconfounded setting, it is feasible to leverage multi-step future observations to relax specific rank conditions in their proposal, which is found to be difficult in the confounded setting \citep{nair2021spectral}. %
Refer to Example~\ref{exa:tabular}.   %

\vspace{-0.15cm}
\textbf{Online learning in POMDPs.} \quad In the literature, statistically efficient online learning algorithms with polynomial sample complexity have been proposed in tabular POMDPs 
\citep{guo2016pac,azizzadenesheli2016reinforcement,jin2020sample,liu2022partially}, %
linear quadratic Gaussian setting (LQG) \citep{lale2021adaptive,simchowitz2020improper}, latent POMDPs \citep{kwon2021rl} and reactive POMDPs/PSRs \citep{krishnamurthy2016pac,jiang2017contextual}. All the methods above require certain model assumptions. To provide a more unified framework, researchers have actively investigated online learning in POMDPs with general function approximation, as evidenced by a vast body of work \citep{zhan2022pac,liu2022optimistic,chen2022partially,zhong2022posterior}. As the most relevant work, by leveraging future-dependent value functions, \citep{uehara2022provably} propose an efficient PAC RL algorithm in the online setting. They require the existence of future-dependent value functions and a low-rank property of the Bellman loss. Instead, in our approach, while we similarly require the existence of future-dependent value functions, we do not require the low-rank property of Bellman loss. Instead, we use the invertibility condition in Theorem~\ref{thm:identify}, i.e., we have informative history proxies as instrumental variables. As a result, the strategies to ensure efficient PAC RL in POMDPs are quite different in the offline and online settings.

\vspace{-0.15cm}
\textbf{Learning dynamical systems via spectral learning.}  \quad There is a rich literature on POMDPs by representing them as predictive state representations (PSRs) \citep{jaeger2000observable,littman2001predictive,rosencrantz2004learning,singh2004predictive}. PSRs are models of dynamical systems that represent the state as a vector of predictions about future observable events. They are appealing because they can be defined directly from observable data without inferring hidden variables and are more expressive than Hidden Markov Models (HMMs) and POMDPs \citep{singh2004predictive}. Several spectral learning algorithms have been proposed for PSRs \citep{boots2011closing,hsu2012spectral,boots2013hilbert}, utilizing conditional mean embeddings or Hilbert space embeddings. These approaches provide closed-form learning solutions, which simplifies computation compared to EM-type approaches prone to local optima and non-singularity issues.  %
While these methods have demonstrated success in real-world applications  \citep{boots2011closing} and have seen subsequent improvements  \citep{kulesza2015spectral,hefny2015supervised,venkatraman2016online,downey2017predictive}, it is still unclear how to incorporate more flexible function approximations such as neural networks in a model-free end-to-end manner with a valid PAC guarantee. %
Our proposed approach not only offers a new model-free method for OPE, but it also incorporates these model-based spectral learning methods when specialized to dynamical system learning with linear models (see \pref{sec:modeling}).

\vspace{-0.15cm}
\paragraph{Planning in POMDPs.} 

There is a large amount of literature on planning in POMDPs. Even if the models are known, exact (or nearly optimal) planning in POMDPs is known to be %
NP-hard in the sense that it requires exponential time complexity with respect to the horizon %
\citep{papadimitriou1987complexity,burago1996complexity}. This computational challenge is often referred to as the \emph{curse of history} \citep{pineau2006anytime}. A natural idea to mitigate this issue is to restrict our attention to %
short-memory policies \citep{azizzadenesheli2018policy,mcdonald2020exponential,kara2022near,golowich2022planning}. Practically, %
a short-memory policy can achieve good empirical performance, as mentioned earlier. %
This motivates us to focus on %
evaluating short-memory policies in this paper. %

\vspace{-0.15cm}
\section{Preliminaries}\label{sec:preliminary}
\vspace{-0.15cm}

In this section, we introduce the model setup, describe the offline data and present some notations. %

\vspace{-0.15cm}
\paragraph{Model setup.} We consider an infinite-horizon discounted POMDP  $\Mcal = \langle \Scal,\Acal,\Ocal, r,\gamma, \OO, \TT \rangle$ where $\Scal$ denotes the state space, $\Acal$ denotes the action space, $\Ocal$ denotes the observation space, $\OO: \Scal\to \Delta(\Ocal)$ denotes the emission kernel (i.e., the conditional distribution of the observation given the state), $\TT:\Scal \times \Acal \to \Delta(\Scal)$ denotes the state transition kernel (i.e., the conditional distribution of the next state given the current state-action pair), $r: \Scal \times \Acal \to \RR$ denotes the reward function, and $\gamma \in [0,1)$ is the discount factor. All the three functions $\TT,r,\OO$ are unknown to the learner.

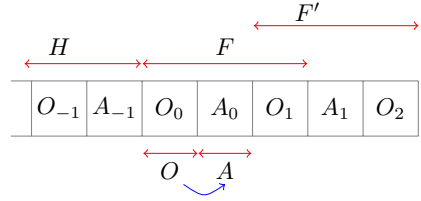
\begin{wrapfigure}{!t}{0.5\textwidth}
\centering
\resizebox{0.4\textwidth}{!}{
\begin{tikzpicture}
\draw[step=0.8cm,gray,very thin] (-1.9,0) grid (4.0,0.8);
\draw (0.4 cm,0.4cm) node  {$O_0$};
\draw (1.2 cm,0.4cm) node  {$A_0$};
\draw (2.0 cm,0.4cm) node  {$O_1$};
\draw (2.8 cm,0.4cm) node  {$A_1$};
\draw (3.6 cm,0.4cm) node  {$O_2$};

\draw (-0.4 cm,0.4cm) node  {$A_{-1}$};
\draw (-1.2 cm,0.4cm) node  {$O_{-1}$};

\draw[red, <->] (-1.7,1.05) -- (-0.02,1.05); 
\draw (-1.2 cm,1.3cm) node  {$H$};

\draw[red, <->] (0.01,1.05) -- (2.39,1.05); 
\draw (1.2 cm,1.3cm) node  {$F$};

\draw[red, <->] (0.02, -0.25) -- (0.8,-0.25); 
\draw (0.4 cm,-0.5cm) node  {$O$};

\draw[red, <->] (0.82, -0.25) -- (1.58,-0.25); 
\draw (1.2 cm,-0.5cm) node  {$A$};

\draw[red, <->] (1.62, 1.6) -- (4.0, 1.6); 
\draw (2.4 cm,1.8cm) node  {$F'$};

\draw[blue, ->] (0.6,-0.7) .. controls (0.9,-0.9) ..(1.2,-0.7); 

\end{tikzpicture}
}
\caption{Case with $M_H=1,M_F=2$. An action $A$ is generated depending on $O$. The extension to memory-based policy is discussed in \pref{sec:history_based}. }\label{fig:pomdp3}
\vspace{-0.2cm}
\end{wrapfigure}

For simplicity, we consider the evaluation of memory-less policies $\pi: \Ocal \to \Delta(\Acal)$ that solely depend on the current observation $O$ in the main text. The extension to policies with short-memory is discussed in \pref{sec:history_based}. 

Next, given a memory-less evaluation policy $\pi^e$, we define the parameter of interest, i.e., the policy value. Following a policy $\pi^e$, the data generating process can be described as follows. First, $S_0$ is generated according to some initial distribution $\nu_{\Scal} \in \Delta(\Scal)$. Next, the agent observes $ O_0 \sim \OO(\cdot \mid S_0)$, executes the initial action $A_0 \sim \pi^e(\cdot \mid O_0)$%
, receives a reward $r(S_0,A_0)$, the environment transits to the next state $S_1 \sim \TT(\cdot \mid S_0,A_0)$, and this process repeats. %
Our objective is to estimate 
$$ \textstyle J(\pi^e):=\EE_{\pi^e}\bracks{\sum_{t=0}^{\infty} \gamma^{t} R_t },$$
where the expectation $\EE_{\pi^e}$ is taken by assuming the data trajectory follows the evaluation policy $\pi^e$.

\paragraph{Offline data.} %
We convert the trajectory data generated by a behavior policy $\pi^b:\Ocal \to \Delta(\Acal)$, into a set of history-observation-action-reward-future transition tuples (denoted by $\mathcal{D}_{\mathrm{tra}}$) and a set of initial observations (denoted by $\Dcal_{\mathrm{ini}}$). The first data subset enables us to learn the reward, emission and transition kernels whereas the second data subset allows us to learn the initial observation distribution. 

Specifically, the dataset $\Dcal_{\mathrm{tra}}$ consists of $N$ data tuples  $\{(H^{(i)},O^{(i)},A^{(i)},R^{(i)},F'^{(i)})\}_{i=1}^N$. We use $(H, O, A, R, F')$ to denote a generic history-observation-action-reward-future tuple where $H$ denotes the $M_H$-step historical observations obtained prior to the observation $O$ and $F'$ denotes the $M_F$-step future observations after $(O,A)$ for some integers $M_H,M_F\geq 1$.  %
Specifically, at a given time step $t$ in the data trajectory, 
we use $(O,A,R)$ to denote $(O_t,A_t,R_t)$, and set
\begin{align*}
    H= (O_{t-M_H:t-1 },A_{t-M_H:t-1})\,\,\hbox{and}\,\,F'=(O_{t+1:t+M_F},A_{t+1:t+M_F-1}).
\end{align*}
We additionally set $F=(O_{t:t+M_F-1},A_{t:t+M_F-2})$. Note we use the prime symbol ' to represent the next time step. 
See Figure \ref{fig:pomdp3} for details when we set $t=0$. 

Throughout this paper, we use uppercase letters such as $(H,S,O,A,R,S',F')$ to denote \emph{random variables} in the offline data, and lowercase letters such as $(h,s,o,a,r,s',f')$ to denote their \emph{realizations}, unless stated otherwise. To simplify the presentation, we assume the stationarity of the environment, i.e., the marginal distributions of $(H,S,F)$ and $(H',S',F')$ are identical.

The dataset $\Dcal_{\mathrm{ini}}$ consists of $N'$ data tuples $\{O^{(i)}_{0:M_F-1},A^{(i)}_{0:M_F-2} \}_{i=1}^{N'}$ which is generated as follows: $ S^{}_0 \sim \nu_{\Scal}$, $O^{}_0 \sim \OO(\cdot \mid S^{}_0)$, 
$A^{}_0 \sim \pi^b(\cdot \mid O^{}_0)$, $S^{}_1 \sim \TT(\cdot \mid S^{}_0,A^{}_0),\cdots$, until we observe $O^{(i)}_{M_F-1}$ and $A^{(i)}_{M_F-1}$. We denote its distribution over 
 $\Fcal=\Ocal^{M_F}\times \Acal^{M_F-1}$ by $\nu_{\Fcal}(\cdot)$. 

\begin{remark}[Standard MDPs]
Consider the setting where $S=O$ and $M_H=0,M_F=1$. In that case, we set $H$ to $S$ instead of histories. Then, $\Dcal_{\mathrm{tra}}=\{O^{(i)},A^{(i)},R^{(i)},O'^{(i)}\}$. We often assume $\nu_{\Ocal}$ ($\nu_{\Fcal}$ in our setting) is known. This yields the standard OPE setting in MDPs \citep{chen2019information,uehara2020minimax}.

\end{remark}

\paragraph{Notation.} We streamline the notation as follows. We define a state value function under $\pi^e$: $V^{\pi^e}(s ) \coloneqq \EE_{\pi^e}[\sum_{k=0}^{\infty} \gamma^{k} R_k \mid S_0=s  ] $  for any $s\in \Scal$. Let $d^{\pi^e}_t(\cdot)$ be the marginal distribution of $S_t$ under the policy $\pi^e$. Then, we define the discounted occupancy distribution $d_{\pi^e}(\cdot) \coloneqq (1-\gamma)\sum_{t=0}^{\infty} \gamma^t d^{\pi^e}_t(\cdot) $. We denote the domain of $F,H$ by $\Fcal = (\Ocal\times \Acal)^{M_F-1}\times \Ocal, \Hcal = (\Ocal\times \Acal)^{M_H}$, respectively. The notations $\mathbb{E}, \mathbb{E}_{\Dcal}$ (without any subscripts) represent the population or sample average over the offline data $\Dcal=\Dcal_{\mathrm{tra}}\cup \Dcal_{\mathrm{ini}}$, respectively. We denote the distribution of offline data by $P_{\pi^b}(\cdot)$. We define the marginal density ratio $w_{\pi^e}(S):= d_{\pi^e}(S)/P_{\pi^b}(S)$. Given a matrix $C$, we denote its Moore-Penrose inverse by $C^{+}$ and its smallest singular value by $\sigma_{\min}(C)$. For a given integer $m>0$, let $I_{m}$ denote an $m\times m$ identity matrix. Let $\otimes$ denote outer product and $[T]=\{0,\cdots,T\}$ for any integer $T>0$.

\vspace{-0.1cm}
\section{Identification via Future-dependent Value Functions}\label{sec:identification}
\vspace{-0.1cm}

In this section, we present our proposal to identify policy values under partial observability by introducing future-dependent value functions. We remark that the target policy's value is identifiable from the observed data via SIS or SDR. Nonetheless, as commented earlier in \pref{sec:intro}, these methods suffer from the curse of horizon. Here, we propose an alternative identification approach that can possibly circumvent the curse of horizon. It serves as a building block to motivate the proposed estimation methods in Section \ref{sec:algorithm}. 

In fully observable MDPs, estimating a policy's value is straightforward using the value function-based method $J(\pi^e)=\EE_{s \sim \nu_{\Scal}}[V^{\pi^e}(s)]$. However, in partial observable environments where the latent state is inaccessible, the state value function $V^{\pi^e}(s)$ is unidentifiable.  
To address this challenge, we propose the use of \textit{future-dependent value functions} that are defined based on observed variables and serve a similar purpose to state value functions in MDPs.

\begin{definition}[Future-dependent value functions]\label{def:valuelink}
Future-dependent value functions $g_V \in [\Fcal \to \RR]$  are defined such that the following holds almost surely,
\begin{align*}
     \EE[g_V(F) \mid S] = V^{\pi^e}(S ). 
\end{align*}
Recall that the expectation is taken with respect to the offline data generated by $\pi^b$. 
\end{definition}

{ Crucially, the future-dependent value functions mentioned above may not always exist, and they need not be unique. Existence is a vital assumption in our framework, although we don't insist on uniqueness. We will return to the topics of existence and uniqueness after demonstrating their relevance in offline policy evaluation.} 

Hereafter, we explain the usefulness of future-dependent value functions. Future-dependent value functions are defined as embeddings of value functions on latent states onto multi-step futures. Notice that $J(\pi^e)=\EE_{s \sim \nu_{\Scal}}[V^{\pi^e}(s)]=\EE_{f \sim \nu_{\Fcal}}[g_{V}(f)]$. These future-dependent value functions are useful in the context of OPE as they enable us to accurately estimate the final policy value. However, the future-dependent value function itself cannot be identified since its definition relies on unobserved states. To overcome this challenge, we introduce a learnable counterpart called the \textit{learnable future-dependent value function}. This learnable version is defined based on observed quantities and thus can be identified.

\begin{definition}[Learnable future-dependent value functions]\label{def:learnable}
Define $\mu(O,A):=\pi^e(A\mid O)/\pi^b(A\mid O)$.
Learnable future-dependent value functions $b_V \in [\Fcal \to \RR]$ are defined such that the following  holds almost surely, 
\begin{align}\label{eq:learnable}
    0 = \EE\bracks{\mu(O,A) \{R+\gamma b_V (F')     \} - b_V (F )  \mid H}. 
\end{align}
Recall that the expectation is taken with respect to the offline data generated by $\pi^b$. We denote the set of solutions by $\Bcal_V$. 
\end{definition}

To motivate this definition, we recall that the off-policy Bellman equation in MDPs \citep{dann2014policy} can be expressed as follows:
\begin{align}\label{eq:standard_bellman}
     V^{\pi^e}(S )%
     = \EE\bracks{ \mu(O,A)(R  +  \gamma V^{\pi^e}(S')) \mid S }. 
\end{align}
Then, by the definition of future-dependent value functions and certain conditional independence relations ($F'\perp (O,A) \mid S'  $), we obtain that 
\begin{align}\label{eq:easy}
    0= \EE\bracks{ \mu(O,A)\braces{R  +  \gamma g_V(F')}-  g_V(F ) \mid S }. 
\end{align}
Since $H \perp (O,A,F') \mid S $, taking another conditional expectation given $H$ on both sides yields that
\begin{align}\label{eq:learnable2}
    0= \EE\bracks{ \mu(O,A)\braces{R  +  \gamma g_V(F')}-  g_V(F ) \mid H }. 
\end{align}
Therefore, \eqref{eq:learnable2} can be seen as an off-policy Bellman equation for future-dependent value functions, analogous to the Bellman equation \eqref{eq:standard_bellman} in MDPs. Based on the above discussion, we present the following lemma.

\begin{lemma}\label{lem:lemma_value}
Future-dependent value functions are learnable future-dependent value functions.
\end{lemma}

\begin{remark}[Non-stationary case] When the offline data is non-stationary, i.e, the pdfs of $(H,S)$ and $(H',S')$  are different, we need to additionally require $\EE[g_V(F') \mid S'] = V^{\pi^e}(S' )$ in the definition. %
 \end{remark}

 \begin{remark}[Comparisons with related works]

Similar concepts have been recently proposed in the context of confounded POMDPs \citep[see e.g.,][]{shi2021minimax,bennett2021}. However, our proposal significantly differs from theirs. First, their confounded setting does not cover our unconfounded setting because their behavior policies $\pi^b$ is \emph{not} allowed to depend on current observations as mentioned in Section~\ref{subsec:related}. Secondly, their proposal
heir proposal overlooks the significant aspect of incorporating multi-step future observations, which plays a pivotal role in facilitating the existence 
of future-dependent value functions, as will be discussed in Example~\ref{exa:tabular}. For a detailed discussion, refer to \pref{sec:difference}. 
\end{remark}

Finally, we present a theorem to identify the policy value.

\begin{theorem}[Identification] \label{thm:identify}
Suppose (\ref{thm:identify}a) 
the existence of learnable future-dependent value functions (need not be unique); (\ref{thm:identify}b) the invertiblity condition, i.e., any $g:  \Scal \to \RR$ that satisfies $\EE[g(S)\mid H]=0$ must also satisfy $g(S)=0$ (
i.e., $g(s)=0$ for almost every $s$ that belongs to the support of $S$), (\ref{thm:identify}c) the overlap condition $w_{\pi^e}(S):=d_{\pi^e}(S)/P_{\pi^b}(S)<\infty, \mu(O,A)<\infty$. Then, for any $b_V\in \Bcal_V$, %
     \begin{align}
         J(\pi^e) = \EE_{f \sim \nu_{ \Fcal}}[b_V(f) ]. 
    \end{align}
\end{theorem}

We assume three key conditions: the observability condition (i.e., $\Bcal_V \neq \emptyset$), the invertibility condition, and the overlap condition. We call Condition (\ref{thm:identify}a) the observability condition since it is reduced to the well-known concept of observability in the LQG control theory, as we will see in \pref{sec:hse}. While Condition (\ref{thm:identify}a) itself is concerned with learnable future-dependent value functions, it is implied by the existence of \textit{unlearnable} future-dependent value functions according to \pref{lem:lemma_value} %
which can be verified using Picard's theorem %
in functional analysis \citep{carrasco2007linear}. In general, the observability condition requires the future proxy $F$ to contain sufficient information about $S$ and is likely to hold when $F$ consists of enough future observations.
We will see more interpretable conditions in the tabular POMDPs (Example~\ref{exa:tabular}) and  POMDPs with Hilbert space embeddings  (HSE-POMDPs) where the underlying dynamical systems have linear conditional mean embeddings (Example~\ref{exa:hse_pomdps} in \pref{sec:hse}). 

The invertibility condition is imposed to ensure that a learnable future-dependent value function $b_V$ satisfies \eqref{eq:easy} (note that right hand side of Eq.~\ref{eq:easy} is a function of $Z, S$ instead of $H$). Again, we will present more interpretable conditions in the tabular and linear settings below and in \pref{sec:hse}. Roughly speaking, it requires $H$ to retain sufficient information about $S$. In that sense, the history proxy serves as an instrumental variable (IV), which is widely used in economics \citep{horowitz2011applied,newey2013nonparametric} \footnote{The observation that history can serve as an instrumental variable in POMDPs is mentioned in \citep{hefny2015supervised,venkatraman2016online}. However, these works aim to learn the system dynamics instead of policy evaluation or learning. Hence, concepts like future-dependent value functions do not appear in these works.}.  

Finally, the overlap condition is a standard assumption in OPE \citep{uehara2021finite}.

\begin{myexp}[Tabular Setting]\label{exa:tabular}
In the tabular case, abstract assumptions in \pref{thm:identify} are reduced to certain rank assumptions. We first define $\Scal_b = \{s \in \Scal:P_{\pi^b}(s)>0\}$. 
We define a matrix $\Prr_{\pi^b}( \Fb \mid  \Sbb_b) \in \RR^{|\Fcal| \times |\Scal|}$ whose $(i,j)$-th element is $\Prr_{\pi^b}( F=x_i\mid  S=x'_j)$ where $x_i$ is the $i$th element in $\Fcal$, and $x'_j$ is the $j$th element in $\Scal_b$.  We similarly define another matrix $\Prr_{\pi^b}( \Sbb_b, \Hb) $  whose $(i,j)$-th element is $\Prr_{\pi^b}(S=x'_i,  H=x_j^{''})$ where %
$x_j^{''}$ denotes the $j$th element in $\Hb$.

\begin{lemma}[Sufficient conditions for observability and invertibility ] \label{lem:equivalent} (a)
When $\rank(\Prr_{\pi^b}( \Fb \mid  \Sbb_b))=|\Scal_b|$, future-dependent value functions exist. Then, from \pref{lem:lemma_value}, learnable future-dependent value functions exist. (b) The invertiblity is satisfied when $\rank(\Prr_{\pi^b}( \Sbb_b, \Hb))=|\Scal_b|$.
\end{lemma}

The first two conditions require that the cardinalities of $\Fcal$ and $\Hcal$ must be greater than or equal to $\Scal_b$, respectively. 
The proof of Lemma~\ref{lem:equivalent} is straightforward, as integral equations reduce to matrix algebra in the tabular setting. We note that similar conditions have been assumed in the literature on HMMs and POMDPs \citep{song2010hilbert,boots2011closing,boots2013hilbert}. In particular, $\Prr_{\pi^b}(\Ob \mid \Sbb_b)=|\Scal_b|$ has been imposed in previous works \citep{nair2021spectral,shi2021minimax} in the context of confounded POMDPs. Nonetheless, our condition $\Prr_{\pi^b}(\Fb \mid \Sbb_b)=|\Scal_b|$ is strictly weaker when $\Fcal$ includes multi-step future observations, %
demonstrating the advantage of incorporating multi-step future observations compared to utilizing only the current observation. A more detailed discussion can be found in Appendix~\ref{subsec:tabular}.

\end{myexp}

\vspace{-0.1cm}
\section{Estimation with General Function Approximation}\label{sec:algorithm}
\vspace{-0.1cm}

\begin{algorithm}[!t]
\caption{ Minimax OPE on POMDPs }\label{alg:main_version}
\begin{algorithmic}[1] 
    \REQUIRE Dataset $\Dcal$, function classes $\Qcal \subset [ \Fcal \to \RR],\Xi \subset [\Hcal \to \RR]$, hyperparameter $\lambda\geq 0$  
    \STATE $
    \hat b_V = \argmin_{ q \in \Qcal} \max_{\xi \in \Xi} \EE_{\Dcal_{\mathrm{tra}}}[\{\mu(A,O)\{R + \gamma q(  F') \} - q( F) \} \xi(H)- \lambda \xi^2(H)]. $
    \RETURN{$\hat J_{\VM} = \EE_{\mathcal{D}_{\mathrm{ini}}}[\hat b_V( f)]$ }
\end{algorithmic}
\end{algorithm}

In this section, we demonstrate how to estimate the value of a policy based on the results presented in Section~\ref{sec:identification}. We begin by outlining the proposed approach for estimating $b_V(\cdot)$. The key observation is that it satisfies $\EE[  L(b_V,\xi)]=0$ for any $\xi:\Hcal \to \RR$ where $  L(q,\xi )$ is defined as 
\begin{align*}
      L(q,\xi ):= \{\mu(A,O)\{R + \gamma q(  F') \} - q( F) \} \xi(H) 
\end{align*}
for $q: \Fcal \to \RR$ and $\xi:\Hcal \to \RR$. Given some constrained function classes $\Qcal \subset  [ \Fcal \to \RR] $ and $\Xi \subset [\Hcal \to \RR]$ and a hyperparameter $\lambda\geq 0$, the estimator is computed according to Line 1 of \pref{alg:main_version}.  When the realizability $\Bcal_V \cap \Qcal \neq \emptyset$ holds and $\Xi$ is unconstrained, we can easily show that the population-level minimizers $  \argmin_{ q \in \Qcal}\max_{\xi \in \Xi} \EE[L(q,\xi)- 0.5 \lambda \xi^2(H)] $ are all learnable future-dependent value functions. This is later formalized in  \pref{thm:bellman}. 

We can use any function classes such as neural networks, RKHS, and random forests to parameterize $\Qcal$ and $\Xi$. Here, the function class $\Xi$ plays a critical role in measuring how $q$ deviates from the ground truth $b_V$. 
The hyperparameter $\lambda$ is introduced to obtain a fast convergence rate. We call it a stabilizer instead of a regularizer since $\lambda$ does not need to shrink to zero as $n$ approaches infinity. Note that regularizers are needed %
when %
we penalize the norms of $q$ and $\xi$, 
i.e., 
\begin{align*}
    \hat b_V = \argmin_{ q \in \Qcal}\max_{\xi\in \Xi} \EE_{\Dcal}[L(q,\xi)- 0.5\lambda \xi^2(H)] + 0.5\alpha' \|q\|^2_{\Qcal}- 0.5\alpha \|\xi\|^2_{\Xi},
\end{align*}
for certain function norms $\|\cdot\|_{\Qcal}$ and  $\|\cdot\|_{\Xi}$ and hyperparameters $\alpha',\alpha>0$.

In the remainder of this section, we present three concrete examples using linear models, RKHSs and neural networks. Let $\|\cdot\|_{\Qcal}$ and $\|\cdot\|_{\Xi}$ denote $L^2$-norms when we use linear models and RKHS norms when using RKHSs. 
\begin{myexp}[Linear models]\label{exa:linear}
Suppose $\Qcal,\Xi$ are linear, i.e., $\Gcal =\{\theta^{\top} \phi_{ \Fcal}(\cdot) \mid \theta \in  \RR^{d_{ \Fcal}}\},\Xi =\{\theta^{\top} \phi_{\Hcal}(\cdot) \mid \theta \in \RR^{d_{\Hcal}}\} $ with features $\phi_{ \Fcal}: \Fcal \to \RR^{d_{ \Fcal}}, \phi_{\Hcal}:\Hcal \to \RR^{d_{\Hcal}} $. Then,
\begin{align*}
    &\hat b_V(\cdot) = \phi^{\top}_{ \Fcal}(\cdot)\braces{ \Mb^{\top}_2\{\alpha I_{d_{\Hcal}} + \lambda \Mb_3\}^{-1}\Mb_2+ \alpha' I_{d_{ \Fcal}} }^{-1} \Mb^{\top}_2\{\alpha I_{d_{\Hcal}} + \lambda \Mb_3\}^{-1}\Mb_1 , \\
    &\Mb_1 = \EE_{\Dcal}[ \mu(O,A) R\phi_{\Hcal}(H)  ],\,\Mb_2 = \EE_{\Dcal}[\phi_{\Hcal}(H)\{\phi^{\top}_{ \Fcal}( F) - \gamma  \mu(O,A)\phi^{\top}_{ \Fcal}( F') \} ], \Mb_3=  \EE_{\Dcal}[ \phi_{\Hcal}(H) \phi^{\top}_{\Hcal}(H) ]. 
\end{align*}
When $\alpha'=0,\alpha=0,\lambda=0$, the value estimators boil down to 
\begin{align*}
     \hat b_V(\cdot) =    \phi^{\top}_{ \Fcal}(\cdot)\Mb^{+}_2\Mb_1,\quad \hat J_{\mathrm{VM}} = \EE_{f\sim \nu_{\Fcal}}[\phi^{\top}_{ \Fcal}(f )]\Mb^{+}_2\Mb_1. 
\end{align*}

The above estimators are closely related to the LSTD estimators in MDPs 
\citep{lagoudakis2003least}. Specifically, the off-policy LSTD estimator for state-value functions \citep{dann2014policy,uehara2020minimax} is given by 
\begin{align}\label{eq:existing_formula}
    \EE_{s \sim \nu_{\Scal}}[\phi^{\top}_{\Scal}(s)]\EE_{\Dcal}[ \phi_{\Scal}(S)\{ \phi^{\top}_{\Scal}(S) -\gamma\mu(S,A)\phi^{\top}_{\Scal}(S')\}]^{+}\EE_{\Dcal}[\mu(S,A)R\phi_{\Scal}(S)]. 
\end{align}
Our new proposed estimator $\hat J_{\textrm{VM}}$ is   
\begin{align*}
    \EE_{f \sim \nu_{ \Fcal}}[\phi^{\top}_{\Fcal}(f)]\EE_{\Dcal}[ \phi_{\Hcal}(H)\{ \phi^{\top}_{\Fcal}(F) -\gamma\mu(O,A)\phi^{\top}_{\Fcal}(F')\}]^{+}\EE_{\Dcal}[\mu(O,A)R \phi_{\Hcal}(H)], 
\end{align*}
which is very similar to \eqref{eq:existing_formula}. The critical difference lies in that we use futures (including current observations) and histories as proxies to infer the latent state $S$ under partial observability. 

\end{myexp}

\begin{myexp}[RKHSs]
Let $\Qcal,\Xi$ be RKHSs with kernels $k_{ \Fcal}(\cdot,\cdot): \Fcal \times  \Fcal \to \RR, k_{\Hcal}(\cdot,\cdot):\Hcal \times \Hcal \to \RR$. Then, 
\begin{align*}
    \hat b_V(\cdot) = \kb_{ \Fcal}(\cdot)^{\top} \braces{ \{\Kb'_{ \Fcal}\}^{\top} \{ \alpha I_n + \Kb_{\Hcal}\}^{-1}\Kb'_{ \Fcal}   + \alpha' I_n}^{-1} \{\Kb'_{ \Fcal}\}^{\top} \Kb^{1/2}_{\Hcal}\{ \alpha I_n + \Kb_{\Hcal}\}^{-1}\Kb^{1/2}_{\Hcal} \Yb 
\end{align*}
where $ \Yb \in \RR^n,  k(\cdot)\in \RR^n, \Kb_{\Hcal} \in \RR^{n\times n}, \Kb_{ \Fcal} \in \RR^{n\times n} $ such that
\begin{align*}
         &\{\Kb_{\Hcal}\}_{(i,j)}=k_{\Hcal}(H^{(i)},H^{(j)}),\,  \{\Kb_{ \Fcal}\}_{(i,j)}=k_{ \Fcal}( F^{(i)}, F^{(j)}), \{ \Yb\}_i = \mu(O^{(i)}, A^{(i)}) R^{(i)},  \\ & \{\Kb'_{ \Fcal}\}_{(i,j)}= k_{ \Fcal}( F^{(i)}, F^{(j)})- \gamma k_{ \Fcal}( F'^{(i)}, F^{(j)}) ,  \{\kb_{ \Fcal}(\cdot)\}_{i}= k_{ \Fcal}(\cdot,  F^{(i)} ). 
\end{align*}
\end{myexp}

\begin{myexp}[Neural Networks]\label{exa:neural}
We set $\Qcal$ to a class of neural networks and recommend to set $\Xi$ to a linear or RKHS class so that
the minimax optimization is reduced to single-stage optimization. The resulting optimization problem can be solved by off-the-shelf stochastic gradient descent (SGD) methods. Specifically, when we set $\Xi$ to a linear model, $\hat b_V$ is reduced to
\begin{align*}\textstyle
 \hat b_V &=\argmin_{q \in \mathbb{Q}}Z(q)^{\top}\{\alpha I + \lambda \EE_{\Dcal}[\phi_{\Hcal}(H)\phi^{\top}_{\Hcal}(H)  ]\}^{-1}Z(q),\, \\
  Z(q)&:= \EE_{\Dcal}[\mu(A,O)  \{R + \gamma q( F')\}-q( F)\}\phi_{\Hcal}(H) ]. 
\end{align*} 
When we set $\Xi$ to be an RKHS, $\hat b_V$ is reduced to
\begin{align*}
    \hat b_V &=   \argmin_{q \in \mathbb{Q}}    \{Z'(q)\}^{\top}\Kb^{1/2}_{\Hcal}\{ \alpha I_n + \lambda \Kb_{\Hcal}\}^{-1}\Kb^{1/2}_{\Hcal}Z(q),  \\ 
    \{Z'(q)\}_i &:= \mu(A^{(i)},O^{(i)})  \{R^{(i)} + \gamma q( F'^{(i)})\}-q( F^{(i)}). 
\end{align*}
\end{myexp}

\begin{remark}[Comparison with minimax methods for MDPs]

Our proposed methods are closely related to the minimax learning approach (a.k.a. Bellman residual minimization) for RL \citep{antos2008learning}. More specifically, in MDPs where the evaluation policy is memory-less, \cite{uehara2021finite,tang2020harnessing} proposed minimax learning methods based on the following equations,
\begin{align*}
    \EE[h(S)\{\mu(S,A)\{\gamma V^{\pi^e}(S') + R\}-V^{\pi^e}(S)\} ] =0, \,\,\,\,\forall h:\Scal \to \RR.
\end{align*}
These methods are no longer applicable in POMDPs since we are unable to directly observe the latent state. Although substituting the latent state with the entire history can still provide statistical guarantees, it is susceptible to the curse of horizon.

\end{remark}

\begin{remark}[Modeling of system dynamics]

Our minimax estimators can be extended to learning system dynamics. In particular, for tabular POMDPs, these results are closely related to the literature on spectral learning in HMMs and POMDPs \citep{song2010hilbert,boots2011closing,boots2013hilbert}. Refer to Section~\ref{sec:modeling},\ref{sec:connection}.

\end{remark}

\begin{remark}[Finite horizon setting]
The extension to the finite horizon setting is discussed in \pref{sec:finitez_horizon2}  
\end{remark}

\vspace{-0.1cm}
\section{Finite Sample Results}\label{sec:finite}
\vspace{-0.1cm}

We study the finite sample convergence rate of the proposed estimators in this section. To simplify the technical analysis, we impose three assumptions. First, we 
assume the function classes are bounded, i.e.,   
$\| \Qcal  \|_{\infty} \leq C_{\Qcal}, \|\Xi \|_{\infty}\leq C_{\Xi}$ for some $C_{\Qcal}$, $C_{\Xi}>0$. Second, we assume that the offline data are i.i.d.\footnote{Without the independence assumption, similar results can be similarly established by imposing certain mixing conditions; see e.g., \citep{shi2020statistical,liao2020batch,kallus2022efficiently}.} Third, we assume $\Qcal,\Xi$ are finite hypothesis classes. Meanwhile, our results can be extended to infinite hypothesis classes using the global/local Rademacher complexity theory; see e.g., \citep{uehara2021finite}. To simplify the presentation, following standard convention in OPE, we suppose the initial distribution is known, i.e., $|\mathcal{D}_{\text{ini}}|=\infty$.   

\paragraph{Accuracy of $\hat b_V$.}

We first demonstrate that %
$\hat b_V$ %
consistently estimates the learnable future-dependent value bridge functions. To formalize this, we introduce the following Bellman operators. 

\begin{definition}[Bellman operators]
The Bellman residual operator onto the history is defined as 
\begin{align*}
    \Tcal: [ \Fcal \to \RR] \ni q(\cdot) \mapsto  \EE[\mu(O,A) \{R +\gamma q( F')\}-q( F) \mid H=\cdot ], %
\end{align*}
and the Bellman residual error onto the history is defined as $\EE[ (\Tcal q)^2(H)]$. Similarly, the Bellman residual operator onto the latent state, $\Tcal^{S}$ is defined as 
\begin{align*}
    \Tcal^S: [ \Fcal \to \RR] \ni q(\cdot)  \mapsto  \EE[\mu(O,A) \{R +\gamma q( F')\} -q( F) \mid  S=\cdot ],  %
\end{align*}
and the Bellman residual error onto the latent state is defined as $\EE[ \{\Tcal^S(q)\}^2( S)]$. The conditional expectations equal to zero when $h$ and ${s}$ lie outside the support of $H$ and $ S$, respectively. 
\end{definition}

The Bellman residual error onto the history is zero for any learnable future-dependent value function, i.e., $\EE[(\Tcal b_V)^2(H)]=0$. Thus, this is a suitable measure to assess how well value function-based estimators approximate the true learnable future-dependent value functions.

\begin{theorem}[Finite sample property of $\hat b_V$]\label{thm:bellman}
Set $\lambda>0$. Suppose (2a) $\Bcal_V \cap \Qcal \neq 0$ (realizability) and (2b) $\Tcal \Qcal \subset \Xi$ (Bellman completeness). With probability $1-\delta$, we have $\EE[(\Tcal \hat b_V)^2(H)]^{1/2} \leq c\{1/\lambda+\lambda\} \max(1,C_{\Qcal}, C_{\Xi})  \sqrt{\frac{\ln (  |\Qcal| |\Xi|c /\delta)}{n} } $ where $c$ is some universal constant. 
\end{theorem}

Note (2a) and (2b) are commonly assumed in the literature on MDPs \citep{chen2019information,uehara2021finite} as well. In particular, Bellman completeness means that the function class $\Xi$ is sufficiently rich to capture the Bellman update of functions in $\Qcal$. { For instance, these assumptions are naturally met in HSE-POMDPs.} We require $\lambda>0$ in the statement of Theorem \ref{thm:bellman} to obtain a parametric convergence rate. When $\lambda=0$, although we can obtain a rate of $O_p(n^{-1/4})$, it is unclear whether we can achieve $O_p(n^{-1/2})$. 

\paragraph{Accuracy of $\hat J_{\mathrm{VM}}$.}

We derive the finite sample guarantee for $\hat J_{\mathrm{VM}}$. 
\begin{theorem}[Finite sample property of $\hat J_{\mathrm{VM}}$] \label{thm:final_error2}
Set $\lambda>0$. Suppose (\ref{thm:bellman}a), (\ref{thm:bellman}b), (\ref{thm:bellman}c) any element in $q \in \Qcal$ that satisfies $
 \EE[\{\Tcal^S(q)\}( S) \mid H] = 0 $ also satisfies $\Tcal^S(q)( S)=0$. 
(\ref{thm:bellman}d) the overlap $\mu(O,A)<\infty$ and any element in $q \in \Qcal$ that satisfies $
    \Tcal^S(q)( S) =0$ also satisfies $\Tcal^S(q)( S^{\diamond}) =0$ 
where $S^{\diamond} \sim d_{\pi^e}(s)$. With probability $1-\delta$,  we have 
\begin{align}\label{eq:first_final_error2}
    |J(\pi^e) - \hat J_{\VM}| &\leq  \frac{c(1/\lambda + \lambda)}{(1-\gamma)^2}\max(1,C_{\Qcal}, C_{\Xi}) \mathrm{IV}_1(\Qcal) \mathrm{Dr}_{\Qcal}[d_{\pi^e},P_{\pi^b}] \sqrt{\frac{\ln (  |\Qcal| |\Xi|c /\delta)}{n} }, 
\end{align}
where 
\begin{align}\label{eqn:ratio}
     \mathrm{IV}^2_1(\Qcal)&:= \sup_{\{q \in \Qcal; \EE[\{\Tcal(q)(H)\}^2]\neq 0\} }\frac{\EE[\{\Tcal^{\Scal}(q)( S)\}^2]  }{ \EE[\{\Tcal(q)(H)\}^2] }, \quad \\
     \mathrm{Dr}^2_{\Qcal}[d_{\pi^e},P_{\pi^b}]& :=  \sup_{ \{q\in \Qcal; \EE_{ s \sim P_{\pi^b}}[\{\Tcal^{\Scal}(q)( s)\}^2]\neq 0  \}  } \frac{\EE_{ s \sim d^{\pi^e}}[\{\Tcal^{\Scal}(q)( s)\}^2  ]  }{\EE_{ s \sim P_{\pi^b}}[\{\Tcal^{\Scal}(q)( s)\}^2] }. 
\end{align}

\end{theorem}

On top of (\ref{thm:bellman}a) and (\ref{thm:bellman}b) that are assumed in \pref{thm:bellman}, when $\mathrm{IV}_1(\Qcal)<\infty, \mathrm{Dr}_{\Qcal}(d_{\pi^e},P_{\pi^b})<\infty$, (\ref{thm:bellman}c) and (\ref{thm:bellman}d) hold, %
we have the non-vacuous PAC guarantee. The condition $\mathrm{Dr}_{\Qcal}(d_{\pi^e},P_{\pi^b})<\infty$ and (\ref{thm:bellman}d) are the overlap conditions, which are adaptive to the function class $\Qcal$ and are weaker than (\ref{thm:identify}c). These are also used in MDPs \citep{xie2021bellman}. Here, $\mathrm{Dr}_{\Qcal}(d_{\pi^e},P_{\pi^b})$ is a refined version of the density ratio and satisfies $\mathrm{Dr}_{\Qcal}(d_{\pi^e},P_{\pi^b})\leq w_{\pi^e}(S)$. The condition $\mathrm{IV}_1(\Qcal)<\infty$ and (\ref{thm:bellman}c) are characteristic conditions in POMDPs that quantify how much errors are properly translated from on $H$ to on $S$. Similar assumptions are frequently imposed in IV literature \citep{DikkalaNishanth2020MEoC,chen2012estimation}.

The upper error bound in \eqref{eq:first_final_error2} does not have explicit exponential dependence on the effective horizon $1/(1-\gamma)$. 
In particular, as shown in \pref{sec:hse}, for tabular POMDPs and HSE-POMDPs, the terms $\mathrm{Dr}_{\Qcal}(d_{\pi^e},P_{\pi^b})$ and $\kappa(\Qcal)$ can be reduced to certain condition numbers associated with covariance matrices spanned by feature vectors; see \eqref{leq:linear_ex2} and \eqref{eq:overlap_linear} in Appendix~\ref{sec:hse}. Hence, unlike SIS-based methods, we are able to break the curse of horizon.  

The numbers of future and history proxies included in $F$ and $H$ represents a tradeoff. Specifically, if $F$ contains enough observations, it is likely that (\ref{thm:bellman}a) will hold. Meanwhile, if $H$ contains enough observations, it is more likely that (\ref{thm:bellman}b) will hold. These facts demonstrate the benefits of including a sufficient number of observations in $F$ and $H$. However, the statistical complexities $\ln(|\Qcal||\Xi|)$ will increase with the number of observations in $F$ and $H$. 

{ %
 Lastly, it's worth noting that while our methods effectively address the curse of horizon, they may incur exponential growth concerning the number of future proxies used. This also applies to history proxies, which should be longer than the length of short memory policies. Here, we focus on the explanation of future proxies. For instance, in the tabular case, $\log |\Omega|$ might scale with $|\Ocal|^{M_F}$ when considering $\Omega$ as the set of all functions on $\Ocal^{M_F}$. However, this situation differs significantly from the curse of horizon, a challenge that naive methods like replacing states with the entire history encounter. These methods would necessitate the entire history to achieve Markovianity, whereas we only require a shorter length of future observations to establish the conditions outlined in Theorem~\ref{thm:identify} (existence), which can be much shorter. Specific examples are provided throughout the paper, including Example 1, which discusses the tabular setting and demonstrates that we essentially need as many future proxies as states, as long as there is sufficient statistical dependence between them.
 }

\vspace{-0.15cm}
\section{Experiment} \label{sec:experiment}
\vspace{-0.15cm}

This section empirically evaluates the performance of the proposed method on a synthetic dataset.\footnote{Our code is available at \href{https://github.com/aiueola/neurips2023-future-dependent-ope}{https://github.com/aiueola/neurips2023-future-dependent-ope}}.

\begin{figure}[!t]
  \centering
  \includegraphics[clip, width=0.7\linewidth]{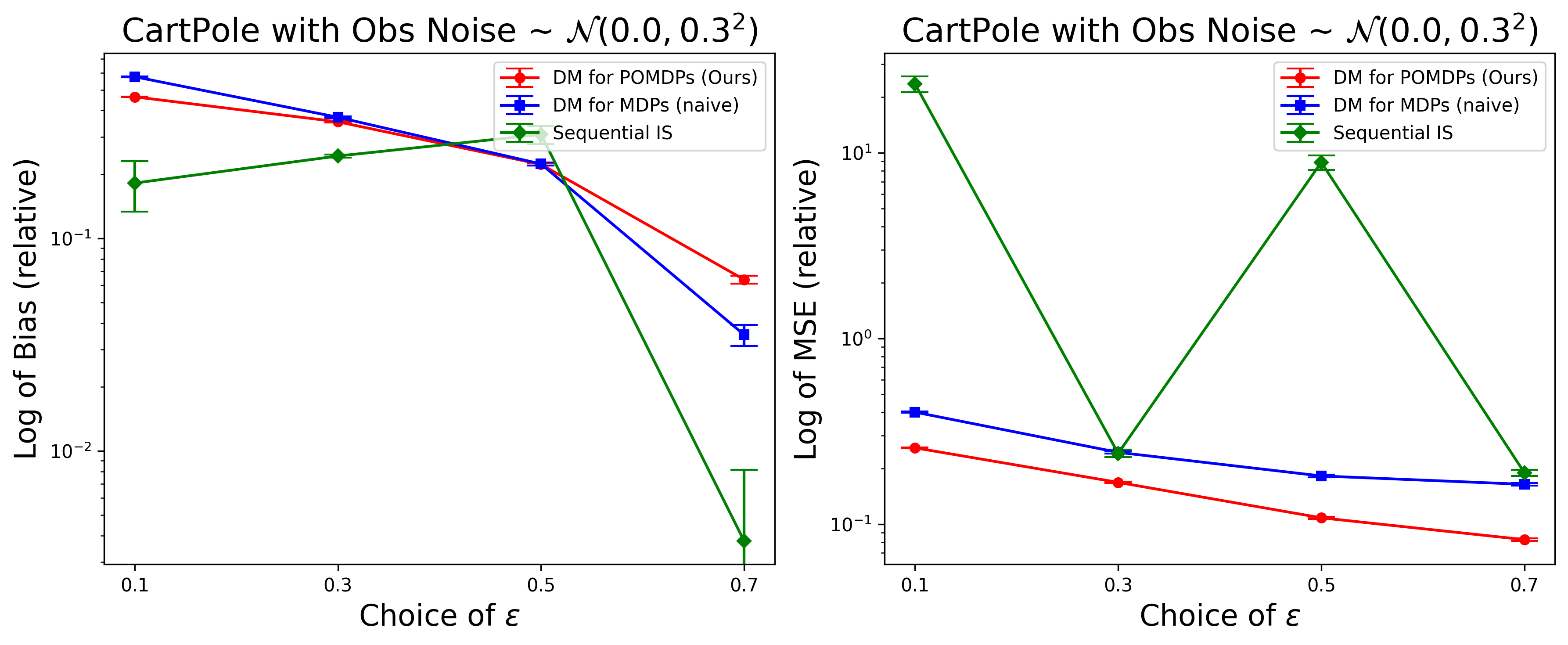}
  \caption{Logarithms of relative biases (left) and MSEs (right) of the proposed and the baseline estimators for various values of $\epsilon$, which specify the evaluation policy. The confidence intervals are obtained through 100 Monte Carlo simulations.} \label{fig:cartpole_main}
  \vspace{-0.3cm}
\end{figure}

We use the CartPole environment provided by OpenAI Gym~\citep{brockman2016openai}, which is commonly employed in other OPE studies \citep{shi2021minimax,farajtabar2018more}. By default, this non-tabular environment consists of 4-dimensional states, which are fully observable. Following~\citep{shi2021minimax}, we create partial observability by adding independent Gaussian noises to each dimension of the state as 
$O^{(j)} = S^{(j)} (1 + \mathcal{N}(1 + 0.3^2)), 1 \leq j \leq 4$. To define behavior and evaluation policies, we first train an expert policy using DDQN~\citep{van2016deep} on latent states $S$. Subsequently, we apply Behavior Cloning (BC) on two datasets, one containing pairs of latent state and action ($S,A$) and the other containing pairs of observation and action ($O,A$), respectively.
Then, we use the base policy obtained by BC on the state-action pairs $(S,A)$ to define an $\epsilon$-greedy behavior policy, where we set $\epsilon = 0.3$. \footnote{When collecting offline data from the behavior policy, the behavior policy takes the observation $O$ as input, noting that the dimension of $S$ and $O$ is the same in our experimental setting. } Similarly, the evaluation policy is also an $\epsilon$-greedy policy, based on the base policy obtained by BC on the observation-action pairs $(O,A)$, with different values of $\epsilon \in [0.1, 0.3, 0.5, 0.7]$. We conduct the experiment with $100$ random seeds, and for each simulation, we collect logged data consisting of $1000$ trajectories, each containing $100$ steps.

We compare our proposal with Sequential Importance Sampling (SIS)~\citep{precup2000eligibility} and the naive minimax OPE~\citep{uehara2020minimax}, which is designed for fully-observable MDPs and does not account for partial observability. The naive minimax estimator is defined as if the environment was fully observable, %
replacing $H$ and $\bar{F}$ in Algorithm~\ref{alg:main_version} with the current observation $O$. In contrast, our proposed method uses a 3-step history as $H$ and a one-step future as $F$ to address partial observability. Both our proposed method and the naive approach use two-layer neural networks for the function $\Qcal$ and RKHSs for $\Xi$, as detailed in Example~\ref{exa:neural}.

We present the results in Figure~\ref{fig:cartpole_main}, which demonstrate the superior accuracy of our proposed estimator compared to the baselines (SIS and naive minimax estimator) in terms of mean square errors (MSEs). Additional experimental details and ablation results, including the variations in the length of $H$, $\bar{F}$, and the choice of RKHSs, can be found in Appendix~\ref{app:cartpole}.

\vspace{-0.1cm}
\section{Conclusion}
\vspace{-0.2cm}

We present a novel approach for OPE of short-memory policies in POMDPs. Our method involves introducing future-dependent value functions and the associated off-policy Bellman equations, followed by proposing a minimax estimator based on these equations. This is the first model-free method that allows for general function approximation and mitigates the curse of horizon. {Our proposal is grounded in three interpretable key assumptions: observability, which asserts the presence of (short) future observations retaining adequate information about latent states, invertibility, which posits the existence of (short) histories preserving ample information about latent states; and the overlap between evaluation policies and behavior policies.}

{ We have several avenues for enhancing our proposals. Firstly, automatically determining the appropriate lengths of futures and histories holds practical significance. Additionally, exploring recent attention mechanisms that extend beyond selecting the most recent history proxies or the nearest future proxies shows promise. Secondly, while we establish Bellman equations for POMDPs and use a simple minimax approach akin to \citep{DikkalaNishanth2020MEoC}, we may benefit from leveraging more refined methods introduced in recent research for solving conditional moment equations \citep{bennett2023source,bennett2023inference,pmlr-v195-bennett23b}. } 

\section*{Acknowledgements}

This material is based upon work supported by the National Science Foundation under Grant Nos. IIS 2112471, IIS 2141781, IIS 1846210, IIS 2154711 and an Engineering and Physical Sciences Research Council grant EP/W014971/1.

\bibliography{rl}
\bibliographystyle{alpha}

\appendix 
\newpage

\section{Additional Literature review} \label{sec:literature_further}

\paragraph{Minimax learning.} %
Minimax learning (also known as adversarial learning) has been widely applied to a large variety of problems ranging from instrumental variable estimation \citep{NIPS2019_8615,DikkalaNishanth2020MEoC} to policy learning/evaluation in contextual bandits/MDPs \citep{hirshberg2017augmented,ChernozhukovVictor2018LLCR,foster2019orthogonal,feng2019kernel,Liu2018,ChernozhukovVictor2020AEoR,uehara2021finite}. For example, in OPE problems with fully-observable environments, minimax learning methods have been developed to learn q-functions and marginal ratios that are characterized as solutions to certain conditional moment equations \citep{uehara2020minimax}. %
The solutions to these conditional moment equations are uniquely defined. On the other hand, our case is more challenging since the solutions to conditional moment equations are not uniquely defined. Although the solutions are \emph{not} uniquely defined and hence cannot be identified, our estimands, i.e., policy values, \emph{can} be still identified. This requires significantly distinctive analysis, which is not seen in standard IV settings or MDP settings.

\section{Comparison to Analogue of Future-dependent Value Functions}\label{sec:difference}

Analogs of our future-dependent value functions have been introduced in confounded contextual bandits and confounded POMDPs as bridge functions \citep{miao2018identifying,cui2020semiparametric,tennenholtz2020off,shi2021minimax,bennett2021}. These works consider confounded settings where actions depend on unobserved states and introduce bridge functions to deal with confounding. Instead, we introduce future-dependent value functions to deal with the curse of the horizon while there is no confounding issue. Existing definitions of bridge functions in confounded POMDPs do not work in standard POMDPs. In the definition of existing bridge functions, behavior policies cannot depend on observations $O$ since observations $O$ are used as so-called negative controls, which do not affect action $A$ and are not affected by action $A$. In our setting, $O$ does \emph{not} serve as negative controls unlike their works since $A$ clearly depends on $O$. Instead, $O$ just play a role in covariates. See \pref{fig:comparison_whole}. 
Due to this fact, we can further add $F$ as input domains of future-dependent value functions, unlike bridge functions by regarding $F$ as just covariates. This is impossible in the definition of existing bridge functions without further assumptions as mentioned in \citep{nair2021spectral}. In this sense, our setting does \emph{not} fall into the category of the so-called proximal causal inference framework. At the same time, our definition does not work in these confounded settings since \pref{def:learnable} explicitly includes behavior policies.   

We finally remark the observation that history can serve as an instrumental variable in POMDPs is mentioned in \citep{hefny2015supervised,venkatraman2016online}. However, they did not propose future-dependent value functions; their goal is to learn system dynamics.

\begin{figure}[!h]
\centering
\begin{subfigure}[b]{.32\textwidth}
  \centering
 \resizebox{\textwidth}{!}{
 \begin{tikzpicture}
\label{fig:comparison}
\node[draw, circle, text centered, minimum size=1.0cm, line width= 1] (s0) {$S$};
\node[draw, circle, below left=3.0 and 0.5 of s0, text centered, minimum size=1.0cm, line width= 1] (a0) {$A$};
\node[draw, circle, below =1.0 of s0, text centered, minimum size=1.0cm, line width= 1] (x0) {$X$};
\node[draw, circle, below left=1.0 and 1.5 of s0, text centered, minimum size=1.0cm, line width= 1] (h0) {$H$};
\node[draw, circle, below right=1.0 and 1.5 of s0, text centered, minimum size=1.0cm, line width= 1] (b0) {$B$};
\node[draw, circle, below right=3.0 and 0.5 of s0, text centered, minimum size=1.0cm, line width= 1] (r0) {$R$};
\path[->]
(s0) edge (a0) 
(a0) edge (r0) 
(s0) edge (r0) 
(x0) edge (a0)
(x0) edge (r0)
(h0) edge (a0)
(b0) edge (r0)
(h0) edge (s0) 
(h0) edge (x0)
(s0) edge (x0)
(x0) edge (b0)
(s0) edge (b0)
;
\path[<->][red]
(a0) edge node[pos=0.25] {\textbackslash\textbackslash} (b0);
\end{tikzpicture}
}
\caption{Standard negative control setting}
\end{subfigure}%
  \hfill
\begin{subfigure}[b]{.28\textwidth}
 \centering
  \resizebox{\textwidth}{!}{
  \begin{tikzpicture}
\label{fig:comparison2}
\node[draw, circle, text centered, minimum size=1.0cm, line width= 1] (s0) {$S$};
\node[draw, circle, below left=3.0 and 0.5 of s0, text centered, minimum size=1.0cm, line width= 1] (a0) {$A$};
\node[draw, circle, below =1.0 of s0, text centered, minimum size=1.0cm, line width= 1] (x0) {$F$};
\node[draw, circle, below left=1.0 and 1.5 of s0, text centered, minimum size=1.0cm, line width= 1] (h0) {$H$};
\node[draw, circle, below right=3.0 and 0.5 of s0, text centered, minimum size=1.0cm, line width= 1] (r0) {$R$};
\path[->]
(s0) edge (x0) 
(a0) edge (r0) 
(s0) edge (r0) 
(h0) edge (a0)
(h0) edge (s0) 
;
\path[<->]
(x0) edge (a0)
;
\end{tikzpicture}
}
\caption{Our POMDP setting with $\gamma=0$}
\end{subfigure}
  \hfill
\begin{subfigure}[b]{.32\linewidth}
 \centering
  \resizebox{\textwidth}{!}{
  \begin{tikzpicture}
\label{fig:comparison3}
\node[draw, circle, text centered, minimum size=1.0cm, line width= 1] (s0) {$S$};
\node[draw, circle, below left=3.0 and 0.5 of s0, text centered, minimum size=1.0cm, line width= 1] (a0) {$A$};
\node[draw, circle, below left=1.0 and 1.5 of s0, text centered, minimum size=1.0cm, line width= 1] (h0) {$H$};
\node[draw, circle, below right=1.0 and 1.5 of s0, text centered, minimum size=1.0cm, line width= 1] (u0) {$O$};
\node[draw, circle, below right=3.0 and 0.5 of s0, text centered, minimum size=1.0cm, line width= 1] (r0) {$R$};
\path[->]
(s0) edge (u0) 
(s0) edge (a0) 
(a0) edge (r0) 
(s0) edge (r0) 
(h0) edge (a0)
(h0) edge (s0) 
;
;
\path[<->][red]
(a0) edge node[pos=0.25] {\textbackslash\textbackslash} (b0);
\end{tikzpicture}
}
\caption{Confounded POMDP setting with $\gamma=0$}
\end{subfigure}

\caption{Comparisons of three DAGs. \textcolor{red}{\textbackslash\textbackslash} means no arrows. For simplicity, we consider the case with $\gamma=0$, i.e., we do not need the next transitions. The first graph is a graph used in \citep{cui2020semiparametric}. 
Note $H,B,X$ are an action negative control, a reward negative control, and a covariate, respectively. We need $(H,A) \perp B \mid S,X$ and $H \perp Y \mid S,X,A$. The graph is one instance satisfying this condition. The second graph corresponds to the contextual bandit version ($\gamma=0$) of our setting. Future proxies $F$ just serve as a covariate and $H$ serves as an action negative control. There are no nodes that correspond to reward negative controls. The third graph corresponds to the contextual bandit version ($\gamma=0$) of confounded POMDPs  \citep{tennenholtz2020off,shi2021minimax,bennett2021}. A node $O$ corresponds to a reward negative control, and $H$ corresponds to an action negative control that satisfies $(H,A) \perp O \mid S$. Thus, $O$ cannot include futures proxies ($F$) since then we cannot ensure $(H,A) \perp F \mid S$ since there is an arrow from $A$ to $F$.    }
\label{fig:comparison_whole}
\end{figure}
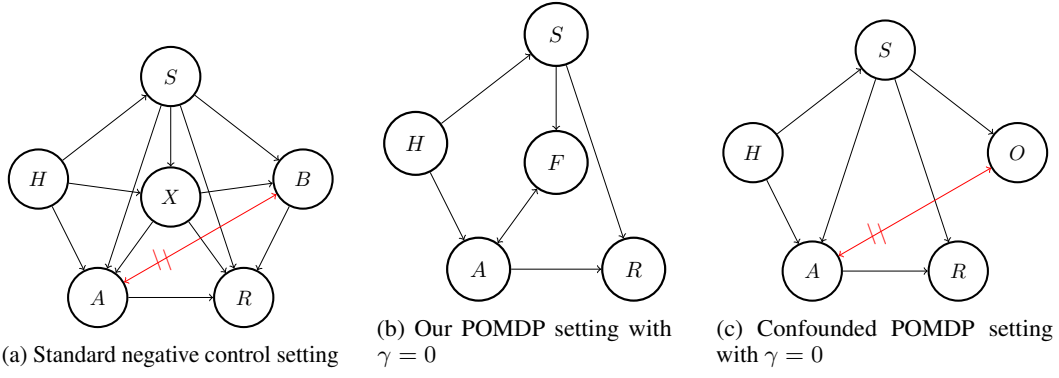

\section{Off-Policy Evaluation for Memory-Based Policies }
\label{sec:history_based}

So far, we have discussed how to evaluate memoryless policies to simplify the notation. In this section, we will now turn our attention to the evaluation of memory-based policies.

\subsection{Settings}

We consider $M$-memory policies $\pi:\Zcal \times \Ocal \to \Delta(\Acal)$ that are functions of the current observation $O_t$ and past observation-action pairs at time point $t,t-1,\cdots,t_M$ denoted by $Z_t=(O_{t-M:t-1},A_{t-M:t-1})\in \Zcal = \Ocal^M \times \Acal^M$, for some integer $M>1$. We assume the existence of $M$ observation pairs obtained prior to the initial time point (denoted by $Z_0$). Following an $M$-memory policy $\pi$, the data generating process can be described as follows. First, $Z_0$ and $S_0$ are generated according to some initial distribution $\nu_{\bar \Scal} \in \Delta(\bar \Scal)$ where $\bar \Scal = \Zcal \times \Scal$. Next, the agent observes $O_0 \sim \mathbb{O}(\cdot \mid S_0)$, executes the initial action $A_0 \sim \pi(\cdot \mid Z_0,O_0)$, receives a reward $r(S_0,A_0)$, the environment transits to the next state $S_1\sim \mathbb{T}(\cdot \mid S_0,A_0)$, and this process repeats. See \pref{fig:pomdp2} for a graphical illustration of the data-generating process. We assume that both the behavior and evaluation policies are $M$-memory. 

\begin{figure}[!h]
\vspace{-0.2cm}
\centering
\resizebox{0.6\textwidth}{!}{
\begin{tikzpicture}
\node[draw, circle, text centered, minimum size=1.0cm, line width= 1] (o0) {$O_{0}$};
\node[draw, circle,  left=0.1 of o0, text centered, minimum size=1.0cm, line width= 1] (a-1) {$A_{-1}$};
\node[draw, circle,  right=0.1 of o0, text centered, minimum size=1.0cm, line width= 1] (a0) {$A_{0}$};
\node[draw, circle,  above=3.0 of a0, text centered, minimum size=1.0cm, line width= 1] (r0) {$R_{0}$};
\node[draw, circle,  right=0.1 of a0, text centered, minimum size=1.0cm, line width= 1] (o1) {$O_{1}$};
\node[draw, circle,  right=0.1 of o1, text centered, minimum size=1.0cm, line width= 1] (a1) {$A_{1}$};
\node[draw, circle,  above=3.0 of a1, text centered, minimum size=1.0cm, line width= 1] (r1) {$R_{1}$};
\node[draw, circle,  right=0.1 of a1, text centered, minimum size=1.0cm, line width= 1] (o2) {$O_2$};
\node[draw, circle,  right=0.1 of o2, text centered, minimum size=1.0cm, line width= 1] (a2) {$A_2$};
\node[draw, circle,  above=3.0 of a2, text centered, minimum size=1.0cm, line width= 1] (r2) {$R_2$};
\node[draw, circle,  fill=gray!40, above=1.0 of o2, text centered, minimum size=1.0cm, line width= 1] (s2) {$S_2$};
\node[draw, circle,  left=0.1 of a-1, text centered, minimum size=1.0cm, line width= 1] (o-1) {$O_{-1}$};
\node[draw, circle,  left=0.1 of o-1, text centered, minimum size=1.0cm, line width= 1] (a-2) {$A_{-2}$};
\node[draw, circle,  left=0.1 of a-2, text centered, minimum size=1.0cm, line width= 1] (o-2) {$O_{-2}$};
\node[draw, circle,  fill=gray!40, above=1 of o0,text centered, minimum size=1.0cm, line width= 1] (s0) {$S_0$};
\node[draw, circle,  fill=gray!40, above=1 of o1,text centered, minimum size=1.0cm, line width= 1] (s1) {$S_1$};

\path[->]
(s0) edge (s1) 
(s0) edge (r0) 
(a0) edge (s1)
(s0) edge (o0)
(a0) edge (r0)
(s1) edge (s2)
(s1) edge (r1)
(s1) edge (o1)
(a1) edge (r1)
(s2) edge (o2)
(a1) edge (s2) 
(s2) edge (r2)
(a2) edge (r2) 
(o-1) edge [out=-80,in=-100,looseness=1] (a0)
(a-1) edge [out=-80,in=-100,looseness=1] (a0)
(o0) edge [out=-80,in=-100,looseness=1] (a0)
(a-2) edge [out=-80,in=-100,looseness=1] (a0)
(o-2) edge [out=-80,in=-100,looseness=1] (a0)

(o1) edge [out=-80,in=-100,looseness=1] (a1)
(a0) edge [out=-80,in=-100,looseness=1] (a1)
(o0) edge [out=-80,in=-100,looseness=1] (a1)
(a-1) edge [out=-80,in=-100,looseness=1] (a1)
(o-1) edge [out=-80,in=-100,looseness=1] (a1)

(o2) edge [out=-80,in=-100,looseness=1] (a2)
(a1) edge [out=-80,in=-100,looseness=1] (a2)
(o1) edge [out=-80,in=-100,looseness=1] (a2)
(a0) edge [out=-80,in=-100,looseness=1] (a2)
(o0) edge [out=-80,in=-100,looseness=1] (a2)
;
\end{tikzpicture}
}
\caption{POMDPs when $M=2$. Note $S_0,O_0,A_0,R_0$ correspond to $(S,O,A,R)$, respectively, and $(O_{-2},A_{-2},O_{-1},A_{-1})= Z_0$. We cannot observe $S$ in the offline data. %
}
\label{fig:pomdp2}
\end{figure}
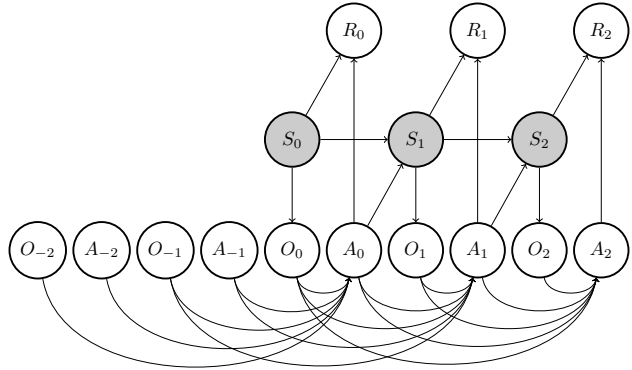

Our goal is to estimate a policy value $J(\pi^e)$ for an $M$-meory evaluation policy $\pi^e$. Toward this end, we define a state-value function under $\pi^e$:
\begin{align*}
    V^{\pi^e}(z,s):=\mathbb{E}_{\pi^e}[\sum_{k=0}^{\infty}\gamma^k R_k \mid Z_0=z,S_0=s]
\end{align*}
for any $z\in \Zcal,s\in \Scal$. Compared to the memory-less case, the input additionally includes $z$. Let $\bar S_t=(Z_t,S_t)$, and $d^{\pi^e}_t(\cdot)$ be the marginal distribution of $\bar S_t$ under the policy $\pi^e$.  

Next, we explain how the offline data is collected when behavior policies are $M$-memory. Specifically, the dataset $\Dcal_{\mathrm{tra}}$ consists of $n$ data tuples  $\{(H^{(i)},O^{(i)},A^{(i)},R^{(i)},F'^{(i)})\}_{i=1}^N$. We use $(H, O, A, R, F')$ to denote a generic history-observation-action-reward-future tuple where $H$ denotes the $M_H$-step historical observations obtained prior to the observation $O$ and $F'$ denotes the $M_F$-step future observations after $(O,A)$ for some integers $M_H>M$ and $M_F\ge 1$.  Hence, given some time step $t$ in the trajectory data, we set $(O,A,R) =(O_t,A_t,R_t)$, 
\begin{align*}
    H= (O_{-M_H:t-1 },A_{-M_H:t-1})\,\,\hbox{and}\,\,F'=(O_{t+1:t+M_F},A_{t+1:t+M_F-1}).
\end{align*}
We additionally set $F=(O_{t:t+M_F-1},A_{t:t+M_F-2})$. We use the prime symbol ' to represent the next time step. Then, 
$Z'=(O_{t-M+1:t}. A_{t-M+1:t})$. See Figure \ref{fig:pomdp4} for details when we set $t=0$. 

\begin{figure}[!t]

\centering
\resizebox{0.6\textwidth}{!}{
\begin{tikzpicture}
\draw[step=0.8cm,gray,very thin] (-4.9,0) grid (4.0,0.8);
\draw (0.4 cm,0.4cm) node  {$O_0$};
\draw (1.2 cm,0.4cm) node  {$A_0$};
\draw (2.0 cm,0.4cm) node  {$O_1$};
\draw (2.8 cm,0.4cm) node  {$A_1$};
\draw (3.6 cm,0.4cm) node  {$O_2$};

\draw (-0.4 cm,0.4cm) node  {$A_{-1}$};
\draw (-1.2 cm,0.4cm) node  {$O_{-1}$};
\draw (-2.0 cm,0.4cm) node  {$A_{-2}$};
\draw (-2.8 cm,0.4cm) node  {$O_{-2}$};
\draw (-3.6 cm,0.4cm) node  {$A_{-3}$};
\draw (-4.4 cm,0.4cm) node  {$O_{-3}$};

\draw[red, <->] (-3.2,-0.25) -- (-0.02,-0.25); 
\draw (-1.7 cm,-0.5cm) node  {$Z$};

\draw[red, <->] (-3.2,-0.75) -- (0.8,-0.75); 
\draw (-1.3 cm,-1.0cm) node  {$(Z,O)$};

\draw[red, <->] (-4.8,1.05) -- (-0.02,1.05); 
\draw (-2.4 cm,1.3cm) node  {$H$};

\draw[red, <->] (0.01,1.05) -- (2.39,1.05); 
\draw (1.2 cm,1.3cm) node  {$F$};

\draw[red, <->] (0.02, -0.25) -- (0.8,-0.25); 
\draw (0.4 cm,-0.5cm) node  {$O$};

\draw[red, <->] (0.82, -0.25) -- (1.58,-0.25); 
\draw (1.2 cm,-0.5cm) node  {$A$};

\draw[red, <->] (1.62, 1.6) -- (4.0, 1.6); 
\draw (2.4 cm,1.8cm) node  {$F'$};

\draw[red, <->] (-1.6,1.6) -- (1.6,1.6); 
\draw (-0.0 cm,1.8cm) node  {$Z'$};

\draw[blue, ->] (-1.3,-1.2) .. controls (-0.4,-1.9) ..(1.2,-0.7); 

\end{tikzpicture}
}
\caption{Case with $M_H=3,M=2,M_F=2$. A 2-memory policy determines action $A$ based on $(Z,O)$.}\label{fig:pomdp4}
\end{figure}
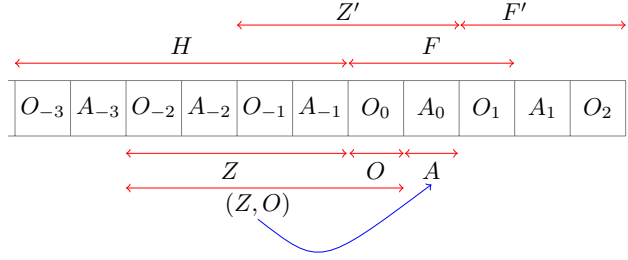

Throughout this paper, uppercase letters such as $(H,S,O,A,R,S',F')$  are reserved for \emph{random variables} and lower case letters such as $(h,s,o,a,r,s',f')$ are reserved for their \emph{realizations}. %
For simplicity, we impose the stationarity assumption, i.e., the marginal distributions of $(H,S,F)$ and $(H',S',F')$ are the same. %

The dataset $\Dcal_{\mathrm{ini}}$ consists of $N'$ data tuples $\{Z^{(i)}_0,O^{(i)}_{0:M_F-1},A^{(i)}_{0:M_F-1} \}_{i=1}^{N'}$ which is generated %
as follows: $\bar S^{}_0 \sim \nu_{\bar \Scal}$, %
$O^{}_0 \sim \OO(\cdot \mid S^{}_0)$, %
$A^{}_0 \sim \pi^b(\cdot \mid Z^{}_0,O^{}_0)$, %
$S^{}_1 \sim \TT(\cdot \mid S^{}_0,A^{}_0),\cdots$, until we observe $O^{(i)}_{M_F-1}$ and $A^{(i)}_{M_F-1}$. We denote its distribution over 
 $\bar \Fcal=\Zcal \times \Fcal$ by $\nu_{\bar \Fcal}(\cdot)$.

\paragraph{Notation.}  We denote the domain of $Z$ by $\Zcal = \Ocal^M \times \Acal^M$. We define $\bar S=(Z,S),\bar F=(Z,F)$. 

\subsection{Required changes in  \pref{sec:identification} }

Every definition and statement holds by replacing $F,S,F',\Fcal,\Scal$ with $\bar F,\bar S,\bar F',\bar \Fcal,\bar \Scal$, respectively. For completeness, we show these definitions and theorems tailored to $M$-memory policies. 

\begin{definition}[Future-dependent value functions]\label{def:valuelink_history}
Future-dependent value functions $g_V \in [\bar \Fcal \to \RR]$  are defined such that the following holds almost surely,
\begin{align*}
     \EE[g_V(\bar F) \mid \bar S] = V^{\pi^e}(\bar S ). 
\end{align*}
\end{definition}

\begin{definition}[Learnable future-dependent value functions]\label{def:learnable_history}
Define $\mu(Z,O,A):=\pi^e(A\mid Z,O)/\pi^b(A\mid Z,O)$.
Learnable future-dependent value functions $b_V \in [\bar \Fcal \to \RR]$ are defined such that the following  holds almost surely, 
\begin{align}
    0 = \EE\bracks{\mu(Z,O,A) \{R+\gamma b_V (\bar F')     \} - b_V (\bar F )  \mid H}. 
\end{align}
We denote the set of solutions by $\Bcal_V$. 
\end{definition}

\begin{theorem}[Identification Theorem] \label{thm:history}
Suppose (a) the existence of learnable future-dependent value functions (need not be unique); (b) the invertiblity condition, i.e., any $g: \bar \Scal \to \RR$ that satisfies $\EE[g(\bar S)\mid H]=0$ must also satisfy $g(\bar S)=0$ (i.e., $g(\bar s)=0$ for almost every $\bar s$ that belongs to the support of $\bar S$), (c) the overlap condition $w_{\pi^e}(\bar S)<\infty, \mu(Z,O,A)<\infty$. Then, for any $b_V\in \Bcal_V$, we have 
     \begin{align}
         J(\pi^e) = \EE_{\bar f \sim \nu_{ \bar \Fcal}}[b_V(\bar f) ]. 
    \end{align}
\end{theorem}

\subsection{Required changes in \pref{sec:algorithm} }

Every algorithm holds by replacing $F,S,F',\Fcal,\Scal$ with $\bar F,\bar S,\bar F',\bar \Fcal,\bar \Scal$, respectively. For completeness, we show the modified version of \pref{alg:main_version} in \pref{alg:main_version_history}.

\begin{algorithm}[!t]
\caption{ Minimax OPE on POMDPs }\label{alg:main_version_history}
\begin{algorithmic}[1] 
    \REQUIRE Dataset $\Dcal$, function classes $\Qcal \subset [ \bar \Fcal \to \RR],\Xi \subset [\Hcal \to \RR]$, hyperparameter $\lambda\geq 0$  
    \STATE $
    \hat b_V = \argmin_{ q \in \Qcal} \max_{\xi \in \Xi} \EE_{\Dcal}[\{\mu(Z,A,O)\{R + \gamma q( \bar F') \} - q(\bar F) \} \xi(H)- \lambda \xi^2(H)]. $
    \RETURN{$\hat J_{\VM} = \EE_{\mathcal{D}_{\mathrm{ini}}}[\hat b_V(\bar f)]$ }
\end{algorithmic}
\end{algorithm}

\subsection{Required change in \pref{sec:finite}}

We present the modified version of \pref{thm:final_error2} tailored to $M$-memory policies. 

\begin{definition}[Bellman operators]
The Bellman residual operator onto the history is defined as 
\begin{align*}
    \Tcal: [\bar \Fcal \to \RR] \ni q(\cdot) \mapsto  \EE[\mu(Z,O,A) \{R +\gamma q(\bar F')\}-q(\bar F) \mid  H=\cdot ], %
\end{align*}
and the Bellman residual error onto the history is defined as $\EE[ (\Tcal q)^2(H)]$. Similarly, the Bellman residual operator onto the latent state, $\Tcal^{S}$ is defined as 
\begin{align*}
    \Tcal^S: [\bar \Fcal \to \RR] \ni q(\cdot)  \mapsto  \EE[\mu(Z,O,A) \{R +\gamma q(\bar F')\} -q(\bar F) \mid \bar S=\cdot ],  %
\end{align*}
and the Bellman residual error onto the latent state is defined as $\EE[ \{\Tcal^S(q)\}^2(\bar S)]$.
\end{definition}

\begin{theorem}[Finite sample property of $\hat b_V$]\label{thm:bellman2}
Set $\lambda>0$. Suppose (\ref{thm:bellman2}a) $\Bcal_V \cap \Qcal \neq 0$ (realizability) and (\ref{thm:bellman2}b) $\Tcal \Qcal \subset \Xi$ (Bellman completeness). With probability $1-\delta$, $$\mathbb{E}[\{\Tcal \hat b_V\}^2(H)]  \leq c\{1/\lambda+\lambda\} \max(1,C_{\Qcal}, C_{\Xi})  \sqrt{\frac{\ln (  |\Qcal| |\Xi|c /\delta)}{n} }, $$ where $c$ is some universal constant. 
\end{theorem}

\begin{theorem}[Finite sample property of $\hat J_{\mathrm{VM}}$] %
Set $\lambda>0$. Suppose (\ref{thm:bellman2}a), (\ref{thm:bellman2}b), (\ref{thm:bellman2}c) any element in $q \in \Qcal$ that satisfies $
 \EE[\{\Tcal^S(q)\}(\bar S) \mid H] = 0 $ also satisfies $\Tcal^S(q)(\bar S)=0$. 
(\ref{thm:bellman2}d) the overlap $\mu(Z,O,A)<\infty$ and any element in $q \in \Qcal$ that satisfies $
    \Tcal^S(q)(\bar S) =0$ also satisfies $\Tcal^S(q)(\bar S^{\diamond}) =0$ 
where $\bar S^{\diamond} \sim d_{\pi^e}(\bar s)$. With probability $1-\delta$,  we have 
\begin{align}
    |J(\pi^e) - \hat J_{\VM}| &\leq  c(1-\gamma)^{-2}(1/\lambda + \lambda)\max(1,C_{\Qcal}, C_{\Xi}) \mathrm{IV}_1(\Qcal) \mathrm{Dr}_{\Qcal}[d_{\pi^e},P_{\pi^b}] \sqrt{\frac{\ln (  |\Qcal| |\Xi|c /\delta)}{n} }, 
\end{align}
where 
\begin{align}
     \mathrm{IV}^2_1(\Qcal)&:= \sup_{\{q \in \Qcal; \EE[\{\Tcal(q)(H)\}^2]\neq 0\} }\frac{\EE[\{\Tcal^{\Scal}(q)(\bar S)\}^2]  }{ \EE[\{\Tcal(q)(H)\}^2] }, \quad \\
     \mathrm{Dr}^2_{\Qcal}[d_{\pi^e},P_{\pi^b}]& :=  \sup_{ \{q\in \Qcal; \EE_{\bar s \sim P_{\pi^b}}[\{\Tcal^{\Scal}(q)(\bar s)\}^2]\neq 0  \}  } \frac{\EE_{ s \sim d^{\pi^e}}[\{\Tcal^{\Scal}(q)(\bar s)\}^2  ]  }{\EE_{\bar s \sim P_{\pi^b}}[\{\Tcal^{\Scal}(q)(\bar s)\}^2] }. 
\end{align}

\end{theorem}

\section{Examples}\label{sec:hse}

\subsection{Tabular POMDPs}\label{subsec:tabular}

We have seen that in \pref{lem:equivalent},  $\rank(\Prr_{\pi^b}( \Sbb_b, \Hb))=|\Scal_b|$ and $\rank(\Prr_{\pi^b}( \Fb \mid  \Sbb_b))=|\Scal_b|$ are sufficient conditions for the identification in the tabular setting. The following theorem show that the abovementioned two conditions are equivalent to $\rank(\Prr_{\pi^b}( \Fb, \Hb))=|\Scal_b|$.
\begin{lemma}\label{lem:equivalent2}
$\rank(\Prr_{\pi^b}( \Sbb_b, \Hb))=|\Scal_b|$ and $\rank(\Prr_{\pi^b}( \Fb \mid  \Sbb_b))=|\Scal_b|$ holds if and only if $\rank(\Prr_{\pi^b}( \Fb, \Hb))=|\Scal_b|$. 
\end{lemma}

We again make a few remarks. First, %
$\rank(\Prr_{\pi^b}(\Fb, \Hb))=| \Scal|$ is often imposed to model HMMs and POMDPs \citep{hsu2012spectral,boots2011closing}. Here, our condition $\rank(\Prr_{\pi^b}(\Fb, \Hb))=| \Scal_b|$ is weaker than this assumption. We discuss the connection to the aforementioned works in \pref{sec:modeling}. Second, in the literature of online RL, $\rank(\Prr_{\pi^b}(\Fb, \Sbb))=| \Scal|$ is frequently imposed as well \citep{jin2020sample,liu2022partially} although they don't impose assumptions associated with the history proxy $H$. In confounded POMDPs, \cite{nair2021spectral,shi2021minimax} imposed a closely-related assumption, namely, $\rank(\Prr_{\pi^b}(\Fb ,\Hb,a))=| \Scal|$ for any $a \in \Acal$ where $\Prr_{\pi^b}(\Fb ,\Hb,a)$ is a matrix whose $(i,j)$-th element is $\Prr_{\pi^b}(F = x_i ,H = x'_j,A=a)$ ($\Fcal=\{x_i\},\Hcal=\{x'_j\}$). 

\subsection{HSE-POMDPs and LQGs}

In this section, we primarily emphasize the identification results in HSE-POMDPs. Extending these results to the final sample result is straightforward. 

\paragraph{Refined identification theorem. }

First, we describe the population version of \pref{thm:final_error2} as follows.

\begin{theorem}[Refined identification theorem ] \label{thm:identify2}
Suppose (\ref{thm:identify2}a) $\Bcal_V \cap \Qcal \neq \emptyset$, (\ref{thm:identify2}b) any element in $q \in \Qcal$ that satisfies $
 \EE[\{\Tcal^S(q)\}(S) \mid H] = 0 $ also satisfies $\Tcal^S(q)(S)=0$. 
(\ref{thm:identify2}c) the overlap $\mu(O,A)<\infty$ and any element in $q \in \Qcal$ that satisfies $
    \Tcal^S(q)(S) =0$ also satisfies $\Tcal^S(q)(S^{\diamond}) =0$ 
where $S^{\diamond} \sim d_{\pi^e}(s)$. 
Under the above three conditions, for any $b_V \in \Bcal_V \cap \Qcal $, we have $$
         J(\pi^e) = \EE_{f \sim \nu_{\Fcal}}[b_V ( f) ]. $$
\end{theorem}

The proof is deferred to \pref{sec:proof_section_D}

\paragraph{HSE-POMDPS. }

By using \pref{thm:identify2},  we can obtain a useful identification formula when we set $\Qcal,\Xi$ to be  linear models in HSE-POMDPs. We start with the definition.

\begin{myexp}[HSE-POMDPs with linear models]\label{exa:hse_pomdps}
Introduce features $\phi_{\Fcal}: \Fcal \to \RR^{d_{\Fcal}},\phi_{\Scal}: \Scal_b \to \RR^{d_{\Scal}},\phi_{\Hcal}: \Hcal \to \RR^{d_{\Hcal}}$ such that $\|\phi_{\Fcal}(\cdot)\| \leq 1, \|\phi_{\Scal}(\cdot)\| \leq 1,\|\phi_{\Hcal}(\cdot)\| \leq 1$. 
Letting $\Qcal$ and $\Xi$ be linear models, the existence of future-dependent value functions in $\Qcal$ is ensured as follows under certain conditions. Then, the existence of learnable future-dependent value functions in $\Qcal$, (\ref{thm:identify2}a) is automatically satisfied. 

Next, we provide sufficient conditions for the realizability (\ref{thm:identify2}a). 

\begin{lemma}\label{lem:linear_ex}
Suppose (LM1): $ \EE[\phi_{\Fcal}( F) \mid  S]=K_1 \phi_{\Scal}( S )$ for some $K_1 \in \RR^{d_{\Fcal} \times d_{\Scal} } $, (LM2): $V^{\pi^e}( S)$ is linear in $ \phi_{\Scal}( S)$, i.e., $V^{\pi^e}( S) \in \{w^{\top}\phi_{\Scal}( S): w\in \RR^{d_{\Scal}}, \|w\|\leq C_{\mathrm{LM}}\} $ for some $C_{\mathrm{LM}} \in \RR$, (LM3) for any $b \in \RR^{\Scal}$ such that $\|b\|\leq C_{\mathrm{LM}}$, there exists $a \in \RR^{d_{\Fcal}},\|a\| \leq C_{\Qcal}$ such that $a^{\top}K_1 \phi_{\Scal}( S) = b^{\top} \phi_{\Scal}( S)$. Then, future-dependent value functions exist and belong to $\Qcal = \{w^{\top}\phi_{\Fcal}(\cdot): w\in \RR^{d_{\Fcal}}, \|w\|\leq C_{\Qcal} \}$ for some $C_{\Qcal} \in \RR$.  
\end{lemma}

The condition (LM1) requires the existence of a conditional mean embedding operator between $\Fcal$ and $\Scal$. This assumption is widely used to model PSRs, which include POMDPs and HMMs \citep{song2009hilbert,boots2013hilbert}. In addition, assumptions of this type are frequently imposed to model MDPs as well \citep{zanette2020learning,duan2020minimax,chowdhury2020no,hao2021sparse}. 
(LM2) is realizablity on the latent state space. (LM3) says the information of the latent space is not lost on the observation space. The condition $\rank(K_1)=d_{\Scal}$ is a sufficient condition; then, we can take $C_{\Qcal}=C_{\mathrm{LM}}/\sigma_{\min}(K_1)$. In the tabular case, we set $\phi_{\Fcal},\phi_{\Scal},\phi_{\Hcal}$ be one-hot encoding vectors over $\Fcal,\Scal_b,\Hcal$, respectively. Here, we remark that $S$ in $\phi_{\Scal}(S)$ is a random variable in the offline data; thus, $\phi_{\Scal}(\cdot)$ needs to be just defined on the support of the offline data. Hence, (LM3) is satisfied when $\rank(\Prr_{\pi^b}( \Fb \mid  \Sbb_b)) = |\Scal_b|$.  

Next, we see  (\ref{thm:identify2}b)  is satisfied as follows under certain conditions.  

\begin{lemma}\label{lem:linear_ex2}
Suppose (LM1), (LM2), (LM4): $\EE[ \mu(O,A) \phi_{\Scal}( S') \mid  S]$ is linear in $\phi_{\Scal}( S)$ and (LM5): 
\begin{align}\label{leq:linear_ex2}
       \sup_{x \in \RR^d, x^{\top}\EE[ \EE[\phi_{\Scal}( S)\mid H]\EE[\phi_{\Scal}( S)\mid H]^{\top}] x\neq 0  }\frac{ x^{\top}\EE[\phi_{\Scal}( S)\phi_{\Scal}( S)^{\top}] x }{x^{\top}\EE[ \EE[\phi_{\Scal}( S)\mid H]\EE[\phi_{\Scal}( S)\mid H]^{\top}] x}<\infty, 
\end{align}
hold. Then, (\ref{thm:identify2}b)  is satisfied.  
\end{lemma}

Condition (LM4) requires the existence of conditional mean embedding between $S'$ and $ S$ under the distribution induced by a policy $\pi^e$.  The condition (LM5) is satisfied when $\rank(\Prr_{\pi^b}( \Sbb_b, \Hb)) = |\Scal_b|$ in the tabular setting. 

Combining \pref{lem:linear_ex} and \pref{lem:linear_ex2} with \pref{thm:identify2}, we obtain the following handy formula. 
\begin{lemma}[Formula with linear models in HSE-PODMDPs]\label{lem:linear_ex3}
Suppose (LM1)-(LM5), (LM6): there exists a matrix $K_2 \in \mathbb{R}^{d_{\Scal} \times d_{\Hcal}} $ such that $\EE[\phi_{\Scal}( S)\mid H] =  K_2 \phi_{\Hcal}(H)$,
(LM7): $\mu(Z,O,A)<\infty$ and 
\begin{align}\label{eq:overlap_linear}
   \sup_{x \in \RR^d, 0 \neq x^{\top}\EE_{s \sim P_{\pi^b} }[\phi_{\Scal}(s)\phi_{\Scal}(s)^{\top}] x }\frac{ x^{\top}\EE_{s \sim d_{\pi^e} }[\phi_{\Scal}(s)\phi_{\Scal}(s)^{\top} ] x }{x^{\top}\EE_{s \sim P_{\pi^b}}[\phi_{\Scal}(s)\phi_{\Scal}(s)^{\top}] x}<\infty, 
\end{align}
hold. Then, we have 
\begin{align}\label{eqn:Jpielinear}
  J(\pi^e) = \EE_{f \sim \nu_{\Fcal}}[\phi_{\Fcal}(f)]^{\top}\EE[\phi_{\Hcal}(H)\{\phi_{\Fcal}( F)-\gamma \mu(O,A)\phi_{\Fcal}( F')\}^{\top} ]^{+}\EE[\mu(O,A)R \phi_{\Hcal}(H)]. 
\end{align}
\end{lemma}
We imposed two additional assumptions in \pref{lem:linear_ex3}. (LM6) is used to ensure the Bellman completeness assumption $\Tcal \Qcal \subset \Xi$. 
(LM7) is similar to the overlap condition (\ref{thm:identify2}c) in linear models. %
It is characterized based on a relative condition number whose value is smaller than the density ratio. Similar assumptions are %
imposed in offline RL with fully observable MDPs as well \citep{xie2021bellman,zanette2021provable,uehara2021pessimistic}.  

\end{myexp}

\paragraph{LQG. }

Finally, we extend our result to the case of LQG. 

\begin{example}[LQG] \label{exa:lqg}
Linear Quadratic Gaussian (LQG) falls in the category of HSE-POMDPs. Suppose 
\begin{align*}
    s_{t+1} = A s_t + B a_t + \epsilon_{1t},\quad r_t = -s^{\top}_t Q s_t - a^{\top}_t R a_t,\quad o_{t}= C s_t  + \epsilon_{2t} 
\end{align*}
where $A,B,C,Q,R$ are matrices that parametrize models and $\epsilon_{1t}$ and $\epsilon_{2t}$ are Gaussian noises. Consider a linear evaluation policy $\pi^e(a_t \mid o_t, z_t) = \rI(a_t = F [o_t, z^{\top}_t]^{\top})$ for certain matrix $F$. %
Notice that linear policies are commonly used in LQG since the globally optimal policy is linear \citep{bertsekas2012dynamic}. Then, defining $\phi_{\Scal}(s) = (1,\{s\otimes s\}^{\top})^{\top}$, $\phi_{\Fcal}(f)=(1,\{f \otimes f\}^{\top})^{\top}$ and $\phi_{\Hcal}(h)=(1,\{h \otimes h \}^{\top})^{\top}$, the following holds. 

\begin{lemma}\label{lem:lqg}
In LQG, (LM1),(LM2), (LM4) are satisfied. When $C$ is left-invertible, (LM3) holds. 
\end{lemma}
Thus, what we additionally need to assume is only (LM5), (LM6) and (LM7) in LQG. 

\end{example}

\section{Finite Horizon Off-Policy Evaluation} \label{sec:finitez_horizon2}

For completeness, we also consider estimation of $J_T(\pi^e)=\EE_{\pi^e}[\sum_{k=0}^{T-1} \gamma^k R_k]$ 
when the horizon is finite and the system dynamics are nonstationary.  We first define value and learnable future-dependent value functions following \pref{sec:identification}. Let $V^{\pi^e}_t(s) = \EE_{\pi^e}[\sum_{k=t}^{\infty}\gamma^{k-t} R_k \mid S_t=s]$ denote the state value function. %

\begin{definition} [Future-dependent value functions]
For $t \in [T-1]$, future-dependent value functions $\{g^{[t]}_V\}_{t=0}^{T-1}$ are defined as solutions to 
\begin{align*}
 0=\EE[ g^{[t]}_V( F) \mid  S] - V^{\pi^e}_t( S) 
\end{align*}
and  $g^{[T]}_V =0$. We denote the set of $g^{[t]}_V$ by $\Gcal^{[t]}_V$.
\end{definition}

\begin{definition} [Learnable future-dependent value functions]
For $t \in [T-1]$, learnable future-dependent value functions $\{b^{[t]}_V\}_{t=0}^{T-1}$ are defined as solutions to 
\begin{align*}
 \EE[\mu(O,A)\{ R+ \gamma b^{[t+1]}_V( F')\} - b^{[t]}_V( F) \mid H]
\end{align*}
where $ b^{[T]}_V =0$. We denote the set of $b^{[t]}_V$ by $\Bcal^{[t]}_V$.
\end{definition}

We define the following Bellman operator:
\begin{align*}
    \Tcal^{\Scal,t}: \prod_{t=0}^{T-1} [ \Fcal \to \RR ] \ni \{q_t(\cdot)\} \mapsto \EE[ \mu(O,A) \{ R+ \gamma q_{t+1}( F')\}- q_t( F) \mid  S=\cdot]\in  [ \Scal \to \RR ]. 
\end{align*}
We again remark that while the conditional expectation of the offline data is not defined on the outside of $\Scal_b$ (the support of $\Scal$) above, we just set $0$ on the outside of $\Scal_b$. 

Here are the analogs statements of \pref{thm:identify2} in the finite horizon setting. 

\begin{lemma}\label{lem:value_lenarable}
Future-dependent value functions are learnable future-dependent value functions. 
\end{lemma}
\begin{theorem}[Identification for finite horizon OPE]\label{thm:identify4}
Suppose for any $t \in [T-1]$, (\ref{thm:identify4}a) $\Bcal^{[t]}_V \cap \Qcal_t \neq \emptyset$, (\ref{thm:identify4}b) for any $q \in \Qcal_t$ that satisfies $\EE[\{\Tcal^{\Scal,t}(q)\}( S)\mid H]=0$ also satisfies $\{\Tcal^{\Scal,t}(q)\}( S)=0,$ (\ref{thm:identify4}c) overlap $\mu(O,A)<\infty$ and for any $q \in  \Qcal_t$ that satisfies $
     \{\Tcal^{\Scal,t}(q)\}( S)= 0 $ also satisfies  $\{\Tcal^{\Scal,t}(q)\}(S^{\diamond}_t)=0$ 
where $S^{\diamond}_t \sim d^{\pi^e}_t(\cdot)$. Then, for any $b^{[0]}_V \in \Bcal^{[0]}_V \cap \Qcal_0 $, we have
\begin{align*}
   J_T(\pi^e ) = \EE_{ f \sim \nu_{ \Fcal}}[b^{[0]}_V( f)]. 
\end{align*}
\end{theorem}
Here, (\ref{thm:identify4}a), (\ref{thm:identify4}b), (\ref{thm:identify4}c) correspond to (\ref{thm:identify2}a), (\ref{thm:identify2}b), (\ref{thm:identify2}c), respectively. \pref{lem:value_lenarable} is often useful to ensure (\ref{thm:identify4}a). 

In the tabular setting, when we have $\rank(\Prr_{\pi_b}(\Fb,\Hb)) = |\Scal_b|$, conditions (a) and (b) are satisfied. This is the same condition imposed in Example \ref{exa:tabular}. 
When we use linear models for $\Qcal_t$ and $\Xi_t$, we have the following corollary.  This is the finite-horizon version of \pref{lem:linear_ex3}. 

\begin{corollary}[Formula with linear models in HSE-POMDPs ] \label{cor:linear_finite}
Suppose (LM1), (LM2f) $V^{\pi^e}_t( S)$ is linear in $\phi_{ \Scal}( S)$ for any $t \in [T-1]$, (LM3), (LM4),(LM5), (LM6). Then under the overlap $\mathrm{Dr}_{\Qcal}(d^{\pi^e}_t,P_{\pi^b} )<\infty 
$ and $\mu(O,A)<\infty$ for any $t \in [T-1]$. Starting from $\theta_T=0$, we recursively define 
\begin{align*}
          \theta_t =\EE[\phi^{\top}_{\Hcal}(H)\phi_{ \Fcal}( F) ]^{+}\EE[ \mu(O,A)\phi_{\Hcal}(H) \{R + \gamma \theta^{\top}_{t+1}\phi_{ \Fcal}( F')\}].  
\end{align*}
Then, $  J_T(\pi^e ) = \EE_{ f \sim \nu_{ \Fcal}}[\theta^{\top}_0 \phi_{ \Fcal}( f)]. $
\end{corollary}

\begin{remark}[Comparison to SIS]
When we have finite samples, the estimator is defined as in \pref{sec:algorithm}. Then, we can obtain the finite sample result as in \pref{sec:finite}. In this case, we can again possibly circumvent the curse of horizon. The error scales with the marginal ratio $\max_{t \in [T-1]}\max_{s \in \Scal}(d^{\pi^e}_t(s)/P_{\pi_b}(s))^{1/2}$. Compared to SIS, the finite sample error does not directly incur the exponential dependence on the horizon. 
\end{remark}

\section{Modeling of System Dynamics } \label{sec:modeling}

We have so far discussed how to estimate cumulative rewards under evaluation policies. In the literature on POMDPs \citep{song2010hilbert,boots2011closing,boots2013hilbert,kulesza2015spectral}, we are often interested in learning of system dynamics. In this section, we discuss how our methods are extended to achieve this goal. We ignore rewards in this section.  We assume policies are memory-less. %

\subsection{Tabular Setting}
Here, let $S^{\diamond}_0,O^{\diamond}_0,A^{\diamond}_0,\cdots$ be random variables under a memory-less evaluation policy $\pi^e:\Ocal \to \Delta(\Acal)$. 
Following \citep{song2010hilbert,boots2011closing,hsu2012spectral}, we consider two common estimands: 
\begin{align} \label{eq:goal1}
     \Prr_{\pi^e}(o_0,a_0,\cdots,o_{T-1},a_{T-1}) &:=\Prr_{\pi^e}(O^{\diamond}_0 = o_0, A^{\diamond}_0 = a_0,O^{\diamond}_1 = o_1,\cdots), \\ 
      \Prr_{\pi^e}(\Ob_T \mid o_0,a_0,\cdots,a_{T-1})&:=\{\Prr_{\pi^e}(O^{\diamond}_T = x_i \mid O^{\diamond}_0 = o_0, \cdots, A^{\diamond}_{T-1}=a_{T-1})\}_{i=1}^{|\Ocal|}, \label{eq:goal2}
\end{align}
given a sequence $o_0 \in \Ocal,\cdots,a_{T-1} \in \Acal$. Our goal is to estimate \pref{eq:goal1} and \pref{eq:goal2} from the offline data. To simplify the discussion, we first consider the tabular case. If we can model a $|\Ocal|$-dimensional vector $\Pr_{\pi^e}(o_0,\cdots,a_{T-1},\Ob_T) \in \RR^{|\Ocal|}$ where the entry indexed by $x_i \in \Ob$ is $\Prr_{\pi^e}(o_0,\cdots, a_{T-1},x_i)$, the latter estimand is computed by normalizing a vector, i.e., dividing it over the sum of all elements. Therefore, we have $ \Prr_{\pi^e}(\Ob_T \mid o_0,a_0,\cdots,a_{T-1}) \propto \Pr_{\pi^e}(o_0,\cdots,a_{T-1},\Ob_T)$. Hereafter, we consider modeling $\Pr_{\pi^e}(o_0,\cdots,a_{T-1},\Ob_T)$ instead of $\Pr_{\pi^e}(\Ob_T \mid o_0,\cdots,a_{T-1})$.  

To identify estimands without suffering from the curse of horizon, we would like to apply our proposal in the previous sections. %
Toward that end, we set rewards as the product of indicator functions 
\begin{align*}
    \Prr_{\pi^e}(o_0,a_0,\cdots,o_{T-1},a_{T-1})= \EE \prns{ \prod_{k=0}^{T-1}\mathrm{I}(O^{\diamond}_t=o_t,A^{\diamond}_t=a_t)}. 
\end{align*}
Since this is a product but not a summation, we cannot directly use our existing results. %
Nevertheless, the identification strategy is similar.

We first introduce learnable future-dependent value functions. These are analogs of \pref{def:learnable} tailored to the modeling of system dynamics. Let $\Xi \subset [\Hcal \to \RR]$ and $\Qcal \subset [ \Fcal \to \RR]$. Below, we fix $o_0 \in \Ocal,\cdots,a_{T-1} \in \Acal$. 

\begin{definition}[Learnable future-dependent value functions for modeling dynamics]
 Learnable future-dependent value functions $\{b^{[t]}_D \}_{t=0}^{T-1}$ where $b^{[t]}_D: \Fcal \to \RR$, for joint observational probabilities are defined as solutions to 
 \begin{align*}
0\leq t\leq T-1 &; \EE[b^{[t]}_D( F) -\rI(O=o_t,A=a_t)\mu(O,A) b^{[t+1]}_D( F') \mid H ]=0,\\
 & \EE[b^{[T]}_D( F)-1 \mid H] = 0. 
\end{align*}
We denote the set of solutions $b^{[t]}_D $ by $\Bcal^{[t]}_D$. Learnable future-dependent value functions $\{b^{[t]}_P \}_{t=0}^{T-1}$ where $b^{[t]}_P: \Fcal \to \RR^{|\Ocal|}$ for conditional observational probabilities are defined as solutions to
 \begin{align*}
0\leq t\leq T-1 &;  \EE[b^{[t]}_P( F)- \rI(O=o_t,A=a_t)\mu(O,A) b^{[t+1]}_P( F') \mid H]=\mathbf{0}_{|\Ocal|},\\
 &\EE[b^{[T]}_P( F)-\phi_{\Ocal}(O) \mid H ]=\mathbf{0}_{|\Ocal|},  
\end{align*}
where $\phi_{\Ocal}(\cdot)$ is a $|\Ocal|$-dimensional one-hot encoding vector over $\Ocal$ and $\mathbf{0}_{|\Ocal|}$ is a $|\Ocal|$-dimensional vector consisting of $0$. We denote the set of solutions $b^{[t]}_P $ by $\Bcal^{[t]}_P$.  
\end{definition}

Next, we define the Bellman operator. 
\begin{definition}[Bellman operators for modeling systems]
\begin{align*}
   \Tcal^{\Scal}_t: \prod_{t=0}^{T-1} [ \Fcal \to \RR]\ni \{q_t(\cdot)\} \mapsto \EE[q_t( F) - \rI(O=o_t,A=a_t)\mu(O,A)q_{t+1}( \Fcal')\mid  S=\cdot]\in [ \Scal \to \RR]. 
\end{align*}
\end{definition}

Following \pref{thm:identify2}, we can identify estimands. Here, let $\grave d^{\pi}_t(\cdot) \in \Delta( \Scal)$ be a probability density function of $S^{\diamond}_t$ conditional on $O^{\diamond}_0=o_0,\cdots,A^{\diamond}_T=a_{T-1}$. 

\begin{theorem}[Identification of joint probabilities]\label{thm:system_identification1}
Suppose (\ref{thm:system_identification1}a) $\Bcal^{[t]}_D \cap \Qcal_t \neq \emptyset $, (\ref{thm:system_identification1}b) for any $q \in \Qcal_t$ that satisfies $\EE[(\Tcal^{\Scal}_t q)(S) \mid H]=0$ also satisfies $(\Tcal^{\Scal}_t q)(S)=0$,(\ref{thm:system_identification1}c) for any $q \in \Qcal_t$ that satisfies $(\Tcal^{\Scal}_t q)( S)=0$  also satisfies $(\Tcal^{\Scal}_t q)(\grave S_t) =0$ where $\grave S_t \sim \grave d^{\pi}_t(\cdot)$ and $\mu(O,A)<\infty$. Then, for any $ b^{[0]}_D \cap \Qcal_0 \in \Bcal^{[0]}_D$,  we have 
\begin{align*}
      \Prr_{\pi^e}(o_0,a_0\cdots,o_{T-1},a_{T-1}) =  \EE_{ f \sim \nu_{ \Fcal}}[b^{[0]}_D(f)  ]. 
\end{align*}
\end{theorem}
\begin{theorem}[Identification of conditional probabilities]\label{thm:system_identification2}
Suppose (\ref{thm:system_identification2}a) $\Bcal^{[t]}_P \cap \Qcal_t \neq \emptyset $, (\ref{thm:system_identification2}b) for any $q \in \Qcal_t $ that satisfies $\EE[\{\Tcal^{\Scal}_t(q)\}( S)\mid  H]=0$ also satisfies  $(\Tcal^{\Scal}_t q)( S)=0$, (\ref{thm:system_identification2}c) for any $q \in \Qcal_t$ that satisfies $(\Tcal^{\Scal}_t q)( S)=0$ also satisfies $ (\Tcal^{\Scal}_t q)(\grave S_t)$ where $\grave S_t \sim \grave d^{\pi}_t(\cdot)$ and $\mu(O,A)<\infty$. Then, for any $ b^{[0]}_P \cap \Qcal_0 \in \Bcal^{[0]}_P$,  we have 
\begin{align*}
      \Prr_{\pi^e}(o_0,a_0\cdots,o_{T-1},a_{T-1},\Ob_T) =  \EE_{ f \sim \nu_{ \Fcal}}[b^{[0]}_P(f)  ]. 
\end{align*}
\end{theorem}

The following corollary is an immediate application of \pref{thm:system_identification1} and \pref{thm:system_identification2}.

\begin{corollary}[Tabular Models]\label{cor:system_identification}
 Let $\phi_{ \Fcal}(\cdot),\phi_{\Hcal}(\cdot)$ be one-hot encoding vectors over $ \Fcal$ and $\Hcal$, respectively. Suppose $\rank( \Prr( \Fb, \Hb))=| \Scal_b|$ and $\grave d^{\pi^e}_t(\cdot)/P_{\pi^b}(\cdot)<\infty$ for any $t\in [T-1]$ where  $P_{\pi^b}(\cdot)$ is a pdf of $ \Scal$ in the offline data. Then, we have 
\begin{align}\label{eq:tabular_modeling}
    \Prr_{\pi^e}(o_0,a_0,\cdots,o_{T-1},a_{T-1}) &=\Prr_{\pi^b}(\Hb)^{\top}B^{+} \braces{ \prod_{t=T-1}^{0} D_t  B^{+} }C, \\
    \Prr_{\pi^e}(\Ob_T \mid o_0,a_0,\cdots,o_{T-1},a_{T-1}) & \propto \Prr_{\pi^b}(\Ob ,\Hb) B^{+} \braces{ \prod_{t=T-1}^{0} D_t B^{+} }C, 
\end{align}
where $   B = \Prr_{\pi^b}( \Fb, \Hb), D_t = \EE[\rI(O=o_t,A=a_t)\mu(O,A)\phi_{ \Fcal}( F')\phi^{\top}_{\Hcal}(H)], C =\Prr_{\nu_{ \Fcal}}(\Fb).$
\end{corollary}

In particular, when behavior policies are uniform policies and evaluation policies are atomic i.e., $\pi^e_t(a) = \mathrm{I}(a=a_t)$ for some $a_t$ and any $t$, %
we have $D_t = \Prr_{\pi^b}(O=o_t, \Fb',\Hb \mid A=a_t)$. In addition, the rank assumption is reduced to $\rank( \Prr_{\pi^b}(\Fb, \Hb))=|\Scal_b|$. 

\paragraph{Linear models.}

Next, we consider cases where $\Qcal_t,\Xi_t$ are linear models  as in Example \ref{exa:linear}. We first define value functions:
\begin{align*}
    V^{\pi^e}_{D,[t]}(\cdot)  & = \Prr_{\pi^e}(O^{\diamond}_t=o_t,A^{\diamond}_t=a_t,\cdots,O^{\diamond}_T=O_T,A^{\diamond}_T=a_T  \mid S^{\diamond}_t=\cdot),  \\
    V^{\pi^e}_{P,[t]}(\cdot)  & = \{\Prr_{\pi^e}(O^{\diamond}_t=o_t,A^{\diamond}_t=a_t,\cdots,O^{\diamond}_T=x_i \mid S^{\diamond}_t=\cdot)\}_{i=1}^{|\Ocal|} . 
\end{align*}
Then, we can obtain the following formula as in \pref{lem:linear_ex3}. Here, $\mathrm{Dr}_{\Qcal}(\grave d^{\pi^e}_t,P_{\pi^b})$ is the condition number in \pref{lem:linear_ex3}.

\begin{corollary}[Formula with linear models in HSE-POMDPs]\label{cor:linear_system}
Suppose (LM1), (LM2D) $V^{\pi^e}_{D,[t]}( S)$ is linear in $ \phi_{ \Scal}( S)$, (LM3), (LM4D) $\EE[\mu(O,A)\rI(O=o_t,A=a_t) \phi_{ \Fcal}(F') \mid  S],\EE[1 \mid  S]$ is linear in $\phi_{ \Scal}( S)$, (LM5) and (LM6D) $\mathrm{Dr}_{\Qcal}(\grave d^{\pi^e}_t,P_{\pi^b})<\infty$. Then, we have 
\begin{align}\label{eq:linear_modeling}
     \Prr_{\pi^e}(o_0,a_0,\cdots,o_{T-1},a_{T-1}) = \EE[\phi_{\Hcal}(H) ]^{\top} B^+ \braces{ \prod_{t=T-1}^{0} D_t  B^+ }C  
\end{align}
where 
\begin{align*}
     B = \EE[\phi_{ \Fcal}( F ) \phi_{\Hcal}(H)^{\top}],  D_t = \EE[\rI(O=o_t,A=a_t)\mu(O,A)\phi_{ \Fcal}( F) \phi_{\Hcal}(H)^{\top}], C=\EE_{ f\sim \nu_{ \Fcal}}[\phi_{ \Fcal}( f)].
\end{align*}
Suppose (LM1), (LM2P) $V^{\pi^e}_{P,[t]}(\cdot)$ is linear in $\phi_{ \Scal}(\cdot)$, (LM3), (LM4P) $\EE[\mu(O,A)\rI(O_t=o_t,A_t=a_t) \phi_{ \Fcal}(F') \mid  S],\EE[\phi_{\Ocal}( \Ocal) \mid  S]$ is linear in $\phi_{ \Scal}( S)$, (LM5) and (LM6P)  $\mathrm{Dr}_{\Qcal}(\grave d^{\pi^e}_t,P_{\pi^b})<\infty$. Then, 
\begin{align}\label{eq:linear_modeling2}
     \Prr_{\pi^e}(\Ob_T \mid o_0,a_0,\cdots,o_{T-1},a_{T-1}) \propto \EE[\phi_{\Ocal}(O)\phi^{\top}_{\Hcal}(H) ] B^+ \braces{ \prod_{t=T-1}^{0} D_t  B^+ }C. 
\end{align} 

\end{corollary}

When behavior policies are uniform, the formulas in \pref{eq:tabular_modeling} and \eqref{eq:linear_modeling} are essentially the same to those obtained via spectral learning  \citep{hsu2012spectral,boots2011closing}. We emphasize that \pref{eq:tabular_modeling}--\eqref{eq:linear_modeling2} appear to be novel to the literature since we consider the offline setting.

\begin{algorithm}[!t]
\caption{ Minimax Modeling of Dynamics on POMDPs }\label{alg:modeling}
\begin{algorithmic}[1] 
    \REQUIRE Dataset $\Dcal$, function classes $\Qcal \subset [\Fcal \to \RR],\Xi \subset [\Hcal \to \RR]$, hyperparameter $\lambda\geq 0$, Horizon $T$
    \STATE Set 
    \begin{align*}
 \hat b^{[T]}_D = \argmin_{q \in \Qcal} \max_{\xi \in \Xi}\EE_{\Dcal}\bracks{\{1- q(F)\}\xi(H) - 0.5 \lambda\xi^2(H) }. 
    \end{align*}
    \FOR{$t=T-1 $}
    \STATE \begin{align}
 \hat b^{[t]}_D = \argmin_{q \in \Qcal} \max_{\xi \in \Xi}\EE_{\Dcal}\bracks{\{\rI(O=o_t,A=a_t)\mu(O,A) \hat b^{[t+1]}_D(F')- q(F)\}\xi(H) - 0.5 \lambda\xi^2(H) }. 
\end{align}
    \STATE $t \gets t-1$
    \ENDFOR
    \RETURN{$\hat J_{\VM} = \EE_{ f \sim \nu_{ \Fcal} }[\hat b^{[0]}_D( f)]$ }
\end{algorithmic}
\end{algorithm}

\paragraph{General function approximation.}

Finally, we introduce an algorithm to estimate joint probabilities in the tabular case with general function approximation, summarized in \pref{alg:modeling}. The conditional probabilities can be similarly estimated. We remark that function approximation is extremely useful in large-scale RL problems.

\subsection{Non-Tabular Setting}

We have so far focused on the tabular case. In this section, we consider the non-tabular case. Our goal is to estimate joint probabilities $\Prr_{\pi^e}(o_0,a_0,\cdots,a_{T-1})$ in \pref{eq:goal1} and 
\begin{align*}
    \EE_{\pi^e}[ \phi_{\Ocal}(O_T) \mid o_0,a_0,\cdots,a_{T-1}] 
\end{align*}
where $\phi_{\Ocal}:\Ocal \to \RR$. When $\phi_{\Ocal}(\cdot)$ is a one-hot encoding vector, this is equivalent to estimating  $\Prr_{\pi^e}(\Ob_T \mid o_0,a_0,\cdots,a_{T-1})$ in \pref{eq:goal2}.

In the non-tabular case, \pref{thm:system_identification1} and \pref{thm:system_identification2} still hold by defining learnable future-dependent value functions $b^{[t]}: \Fcal \to \RR$  as solutions to 
\begin{align*}
    \EE[b^{[t]}( F) \mid H] = \EE[\mu(O,A) b^{[t]}( F') \mid H,O=o_t,A=a_t]\Prr_{\pi^b}(O=o_t,A=a_t). 
\end{align*}
where $b^{[t]}$ is either $b^{[t]}_D$ or $b^{[t]}_P$. Then, Corollary \ref{cor:linear_system} holds by just replacing $D_t$ with  
\begin{align*}
    \EE[  \mu(O,A)\phi_{ \Fcal}( F)\phi_{\Hcal}^{\top}(H) \mid O=o_t,A=a_t]\Prr_{\pi^b}(O=o_t,A=a_t). 
\end{align*}

When we have a finite sample of data, we need to perform density estimation for $\Prr_{\pi^b}(O=o_t,A=a_t)$. This practically leads to instability of estimators. However, when our goal is just to estimate $   \EE[\phi_{\Ocal}(O) \mid o_0,a_0,\cdots,a_{T-1}] $ up to some scaling constant as in \citep{song2010hilbert}, we can ignore $\Prr_{\pi^b}(O=o_t,A=a_t)$. Then, we obtain the following formula:
{\small 
\begin{align}\label{eq:up_to_constant}
       \EE_{\pi^e}[ \phi_{\Ocal}(O_T) \mid o_0,a_0,\cdots,a_{T-1}]  \propto  \EE[\phi_{\Ocal}(O)\phi^{\top}_{\Hcal}(H) ]B^+\left \{\prod_{t=T-1}^{0} \EE[  \mu(O,A)\phi_{ \Fcal}( F)\phi_{\Hcal}^{\top}(H) \mid O=o_t,A=a_t]  B^+ \right \}C. 
\end{align}
} 
This formula is known in HMMs \citep{song2010hilbert} and in POMDPs \citep{boots2013hilbert} where behavior policies are open-loop. %

\section{Connection with Literature on Spectral Learning}\label{sec:connection}

We discuss the connection with previous literature in detail. \citep{song2010hilbert,hsu2012spectral} consider the modeling of HMMs with no action. In this case, our task is to model
\begin{align*}
    \Prr(o_0,\cdots,o_{T-1}):=   \Prr(O_0=o_0,\cdots,O_{T-1}=o_{T-1}) 
\end{align*}
and the predictive distribution: 
\begin{align*}
    \Prr(\Ob_T \mid o_0,\cdots,o_{T-1}):=   \{\Prr(O_T=x_i \mid O_0 = o_0,\cdots,O_{T-1}=o_{T-1})\}_{i=1}^{|\Ocal|}. 
\end{align*}

In the tabular case, the Corollary \ref{cor:system_identification} is reduced to 
\begin{align*}
    \Prr(o_0,\cdots,o_{T-1}) &=\Prr(\Hb)^{\top}\Prr(\Fb,\Hb)^{+} \braces{ \prod_{t=T-1}^{0} \Prr(O=o_t, \Fb',\Hb) \Prr(\Fb, \Hb)^{+} }\Prr_{\nu_{\Fcal} }(\Fb_0), \\ 
     \Prr(\Ob_T \mid o_0,\cdots,o_{T-1}) & \propto \Prr(\Ob,\Hb)\Prr(\Fb,\Hb)^{+} \braces{ \prod_{t=T-1}^{0} \Prr(O=o_t, \Fb',\Hb) \Prr(\Fb, \Hb)^{+} }\Prr_{\nu_{\Fcal} }(\Fb_0). 
\end{align*}
This formula is reduced to the one in \citep{hsu2012spectral} when $\Fcal = \Hcal$ and  $\Prr(\Hb)=\Prr_{\nu_{\Fcal} }(\Fb_0) $ (stationarity). Here, the offline data consists of three random variables $\{O_{-1}, O_0, O_1\}$. In this case, the above formulae are 
{\small 
\begin{align*}
    \Prr(o_0,\cdots,o_{T-1}) &=\Prr(\Ob_{-1} )^{\top}\Prr(\Ob_0,\Ob_{-1})^{+} \braces{ \prod_{t=T-1}^{0} \Prr(O_0=o_t, \Ob_1,\Ob_{-1}) \Prr(\Ob_0,\Ob_{-1})^{+} }\Prr(\Ob_{-1}), \\ 
     \Prr(\Ob_T \mid o_0,\cdots,o_{T-1}) & \propto \Prr(\Ob_0,\Ob_{-1})\Prr(\Ob_0,\Ob_{-1})^{+} \braces{ \prod_{t=T-1}^{0}\Prr(O_0=o_t, \Ob_1,\Ob_{-1}) \Prr(\Ob_0,\Ob_{-1})^{+} }\Prr(\Ob_{-1}). 
\end{align*}
} 

Next, we consider the case when we use linear models to estimate 
\begin{align*}
     \EE[ \phi_{\Ocal}(O_T) \mid o_0,\cdots,o_{T-1}] 
\end{align*}
up to some constant scaling. When there are no actions, the formula \pref{eq:up_to_constant} reduces to 
{\small 
\begin{align*}\textstyle
\EE[\phi_{\Ocal}(O)\phi^{\top}_{\Hcal}(H) ]\EE[\phi_{\Fcal}(F)\phi^{\top}_{\Hcal}(H) ]^+\left \{\prod_{t=T-1}^{0} \EE[\phi_{\Fcal}(F')\phi_{\Hcal}^{\top}(H) \mid O=o_t]\EE[\phi_{\Fcal}(F)\phi^{\top}_{\Hcal}(H) ]^+ \right \}\EE_{f \sim \nu_{\Fcal}}[\phi_{\Fcal}(f)].
\end{align*}
} 

When $\Fcal = \Hcal$, the pdf of $O_{-1}$ is the same as $\nu_{\Fcal}(\cdot)$ and the offline data consists of three random variables $\{O_{-1}, O_0, O_1\}$, the above is reduced to the one in \citep{song2010hilbert} as follows: 
{\small 
\begin{align*}\textstyle
\EE[\phi_{\Ocal}(O_0)\phi^{\top}_{\Ocal}(O_{-1}) ]\EE[\phi_{\Ocal}(O_0)\phi^{\top}_{\Ocal}(O_{-1}) ]^+\left \{\prod_{t=T-1}^{0} \EE[\phi_{\Ocal}(O_1)\phi_{\Ocal}^{\top}(O_{-1}) \mid O_0=o_t]\EE[\phi_{\Ocal}(O_0)\phi^{\top}_{\Ocal}(O_{-1}) ]^+ \right \}\EE[\phi_{\Ocal}(O_{-1})]. 
\end{align*}
}

\section{Omitted Experiment Details and Additional Results} \label{app:experiment}
This section provides additional implementation details and ablation results of the synthetic experiment.

\subsection{Sequential Importance Sampling (SIS)} We compare SIS~\citep{precup2000eligibility} as a baseline estimator. SIS is a non-parametric approach, which corrects the distribution shift between the behavior and evaluation policies by applying importance sampling as follows.
\begin{align*}
    \hat{J}_{\mathrm{SIS}}(\pi^e; \mathcal{D}) := \mathbb{E}_{\mathcal{D}} \left[ \sum_{t=1}^{\infty} \gamma^{t - 1} \left( \prod_{t'=1}^{t} \frac{\pi^e(O_{t'} \,|\, A_{t'})}{\pi^b(O_{t'} \,|\, A_{t'})} \right) R_t \right]
\end{align*}
SIS enables unbiased estimation when absolute continuity (i.e., $\forall (O, A) \in \mathcal{O} \times \mathcal{A}, \pi^e(O | A) > 0 \rightarrow \pi^b(O | A) > 0$) holds. However, as the importance weight grows exponentially as $t$ becomes large, SIS suffers from high variance~\citep{Liu2018,uehara2020minimax,shi2021minimax}. We also empirically verify that SIS incurs high estimation error due to variance in the experimental results.

\subsection{Evaluation Metrics} We use the same evaluation metrics with \citep{shi2021minimax}. Given the i.i.d. dataset $\mathcal{D}_1, \mathcal{D}_2, \cdots, \mathcal{D}_m$, the values estimated on them $\hat{J}_1, \hat{J}_2, \cdots, \hat{J}_m$, and the true value $J^{\pi^e}$, we define the relative bias and relative MSE as follows:
\begin{align*}
    \mathrm{Bias}(\hat{J}; \mathcal{D}, \pi) := \left| \frac{1}{m} \sum_{j=1}^m \left(\frac{\hat{J}_i - J^{\pi^e}}{J^{\pi^e}} \right) \right|, \quad \quad \quad \mathrm{MSE}(\hat{J}; \mathcal{D}, \pi) := \frac{1}{m} \sum_{j=1}^m \left(\frac{\hat{J}_i - J^{\pi^e}}{J^{\pi^e}} \right)^2
\end{align*}
We use the above metrics to compare the performance of SIS, the naive baseline, and our proposal.

\subsection{CartPole Setting} \label{app:cartpole}

\paragraph{Environment.} 
Here, we describe the state, action, and reward settings of the CartPole environment. First, the 4-dimensional states of CartPole represent the position and velocity of the cart and the angle and angle velocity of the pole. The action space is $\{0, 1\}$, either pushing the cart to the left or right. To better distinguish the values among different policies, we used modified reward following \citep{shi2021minimax}. Specifically, we define the reward as
\begin{align*}
    R = \frac{1}{2} \left(\left|2.0 - \frac{x}{x_{\mathrm{clip}}}\right| \cdot \left|2.0 - \frac{\theta}{\theta_{\mathrm{clip}}}\right| - 1.0 \right),
\end{align*}
where $x$ and $\theta$ are the positions of Cart and angle of Pole. $x_{\mathrm{clip}}$ and $\theta_{\mathrm{clip}}$ are the thresholds such that the episode will terminate when either $|x| \geq x_{\mathrm{clip}}$ or $|\theta| \geq \theta_{\mathrm{clip}}$. Under this definition, we observe a larger reward when the cart is closer to the center, and the pole stands straight. We also set the discount factor $\gamma = 0.95$, and the values of the policies used in our experiment are somewhere between 20 and 40.

\paragraph{Estimator Implementation. } We parametrize the value function $b_v(\cdot)$ of our proposal and the naive estimator with a two-layer neural network. The network uses a 100-dimensional hidden layer with ReLU as its activation function, and Adam~\citep{kingma2014adam} is its optimizer. Both the naive estimator and our proposal optimize $b_v(\cdot)$ with the loss function defined in Example 5, but the naive one replaces $\bar{F}$ and $H$ with $O$.
Specifically, the naive estimator takes 4-dimensional observation $O$ as input. On the other side, our proposed estimator additionally inputs the concatenated vector of observation $O$ and one-hot representation of $A$ for several future steps (i.e., $M_F$) to consider. The convergence is based on the test loss evaluated on the test dataset, which is independent of the datasets to train value functions and estimate the policy value. Specifically, to find a global convergence point, we first run 20000 epochs with 10 gradient steps per each. Then, we report the results on the epoch which achieves the minimum loss. The convergence point is usually less than 10000 epochs.

For the adversarial function space $\Xi$ of both methods, we use the following RBF kernel $K(x_i; x_2)$:
\begin{align*}
    K(x_1; x_2) := \exp \left(- \frac{\| x_1 - x_2 \|_2}{2 \sigma^2} \right),
\end{align*}
where $\|x_1 - x_2\|_2$ is the l2-distance between $x_1$ and $x_2$, and $\sigma$ is a bandwidth hyperparameter. We use $\sigma = 1.0$ in the main text and provide ablation results with varying values of $\sigma$ in the following section. Finally, the naive estimator uses $O$ as the input of the kernel. The proposed method first predicts the latent spaces from the historical observations as $\hat{S} := f_{\mathrm{LSTM}}(H)$ and then use $\hat{S}$ as the input of the kernel. $f_{\mathrm{LSTM}}(H)$ is a bi-directional LSTM~\citep{cui2018deep} with 10-dimensional hidden dimension. We train the LSTM with MSE loss in predicting the noisy state (i.e., $O$) with Adam~\citep{kingma2014adam} as its optimizer.

\begin{figure}[!t]
  \centering
  \includegraphics[clip, width=0.7\linewidth]{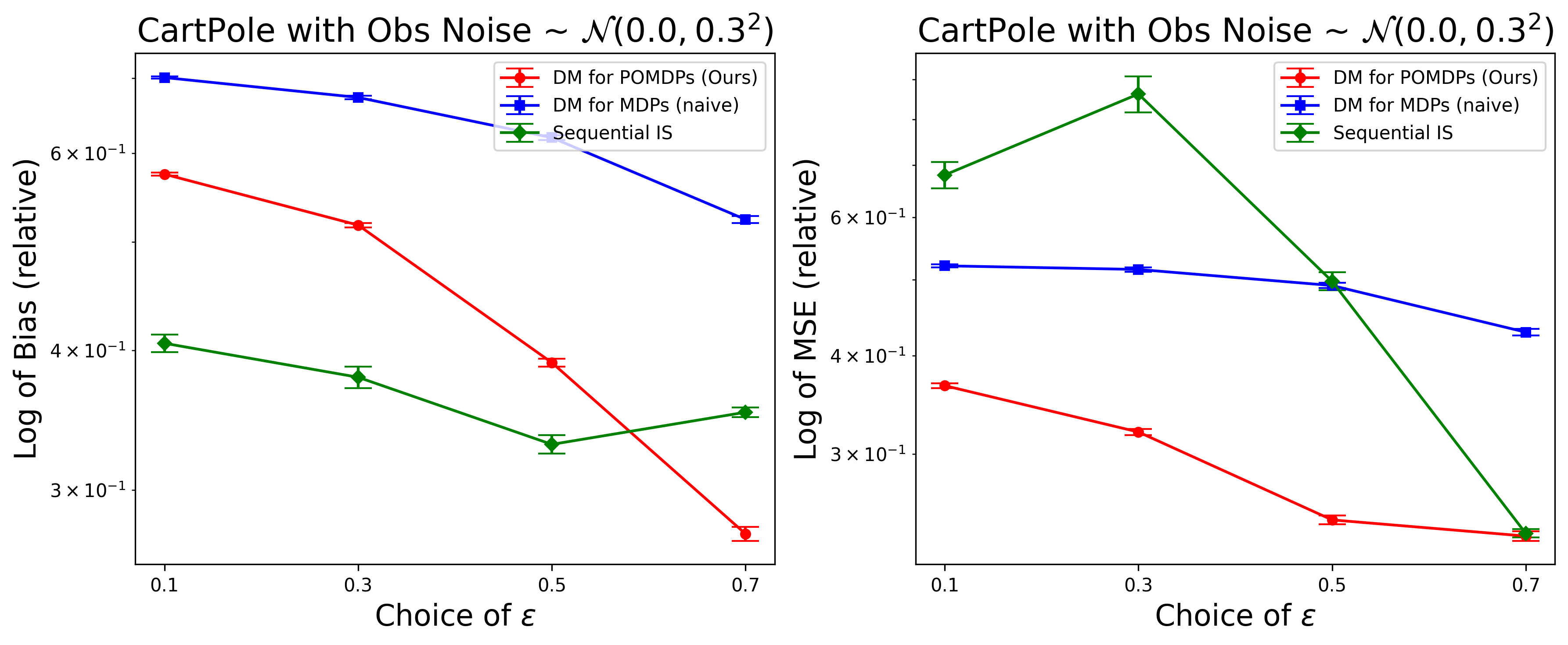}
  \caption{Logarithms of relative biases (left) and MSEs (right) of the proposed and the baseline estimators for various values of $\epsilon$, which specify the evaluation policy. The confidence intervals are obtained through 100 Monte Carlo simulations.} \label{fig:cartpole_01}
\end{figure}

\begin{figure}
  \centering
  \includegraphics[clip, width=0.90\linewidth]{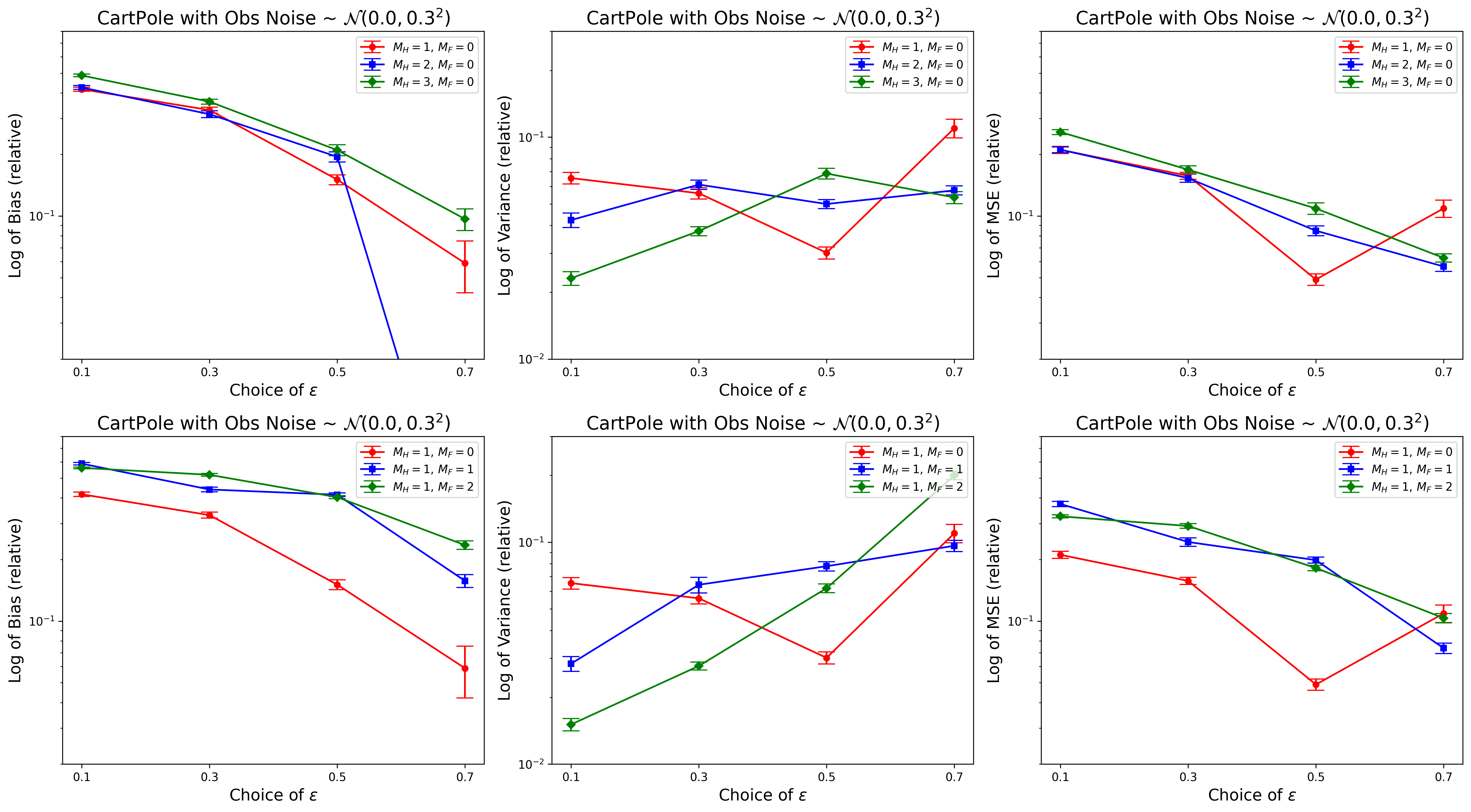}
  \caption{Logarithms of relative biases (left), variances (center),  and MSEs (right) of the proposed estimator with varying lengths of history $M_H$ (top) and varying lengths of future steps $M_F$ (bottom).
  The x-axis corresponds to the varying values of $\epsilon$ of the evaluation policy, and the
  associated confidence interval is based on 20 simulations.} \label{fig:cartpole_ablation}
\end{figure}

\begin{figure}
  \centering
  \includegraphics[clip, width=0.90\linewidth]{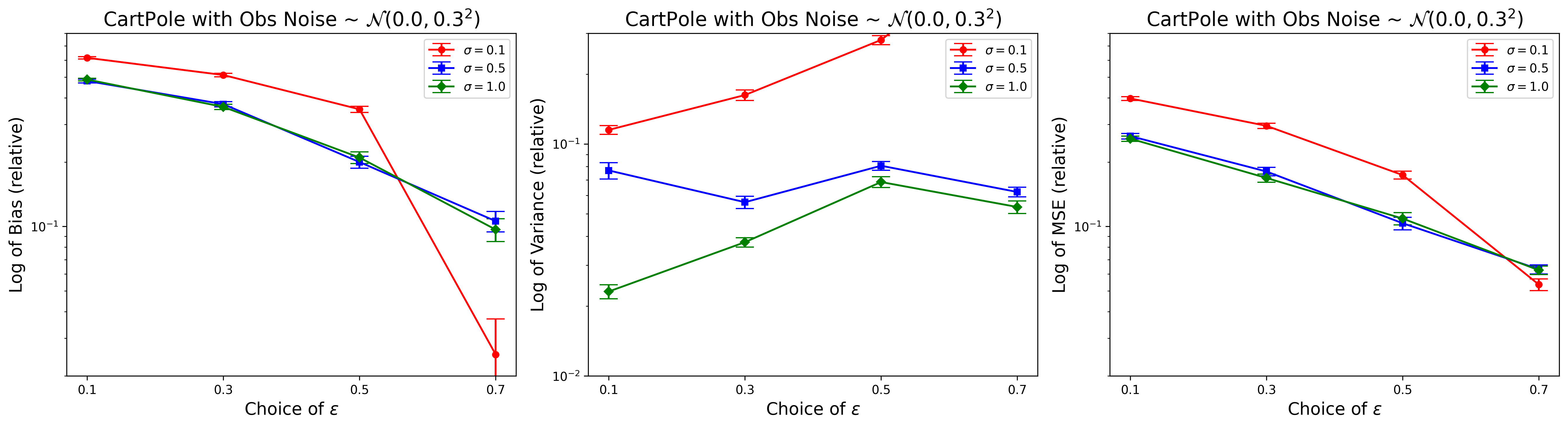}
  \caption{Logarithms of relative biases (left), variances (center),  and MSEs (right) of the proposed estimator with varying bandwidth hyperparameter $\sigma$ of RKHSs.
  The x-axis corresponds to the varying values of $\epsilon$ of the evaluation policy, and the
  associated confidence interval is based on 20 simulations.} \label{fig:cartpole_ablation_sigma}
\end{figure}

\subsection{Additional Rsults}
Figure~\ref{fig:cartpole_01} shows the experimental results in the case of using the behavior policy with $\epsilon=0.1$. The result suggests that the proposed estimator reduces MSEs of the baseline estimators as observed in the $\epsilon=0.3$ case in the main text. Note that experimental settings other than $\epsilon$ of the behavior policy are the same as those used in Section~\ref{sec:experiment}. 

\subsection{Ablation Results}
Here, we provide ablation results with (1) varying choices of $\bar{F}$ and $H$, and (2) varying values of bandwidth hyperparameter $\sigma$ of RKHSs. We provide the results with 20 random seeds in the following to conduct comprehensive ablations.

The first set of experiments aims to study how the use of history and future observations help improve the accuracy of value estimation. For this, we compare our proposed method with varying values of history length $M_H$ and future length $M_F$. Figure~\ref{fig:cartpole_ablation} (top) shows the result of varying history lengths $M_H \in \{1, 2, 3\}$ with a fixed future length ($M_F = 0$). We observe that the performance of the proposed method does not change greatly with the choice of history lengths. This result suggests that 1-step history is almost sufficient to identify the latent state in our experimental settings. Next, we report the results with varying future lengths $M_F \in \{ 0, 1, 2 \}$ with a fixed history length ($M_H = 1$)  in Figure~\ref{fig:cartpole_ablation} (bottom). The result suggests that the increased number of future steps can slightly increase bias. We attribute this to the increased estimation difficulty of the value function due to the increase in the dimensionality of inputs. However, we should also note that future observations may help improve the performance of the proposed method when the history itself is insufficient to identify the latent state.

The second set of experiments is to see how robust the proposed method is to the choice of hyperparameter $\sigma$. We thus vary the values of $\sigma \in [0.1, 0.5, 1.0]$, and report the result in Figure~\ref{fig:cartpole_ablation_sigma}. The result shows that the estimation accuracy is almost the same between $\sigma=1.0$ and $\sigma=0.5$, suggesting that the proposed value learning method is robust to the change of bandwidth hyperparameter of RKHSs to some extent. On the other hand, we observe that a very small value (i.e., $\sigma=0.1$) can increase the variance of estimation. Therefore, our recommendation is to avoid a (too) small value for the bandwidth hyperparameter $\sigma$.

\section{Proof of \pref{sec:identification}}

We often use
\begin{align*}
    F' \perp S,O,A  \mid S'\quad \mathrm{and} \quad (H ) \perp (A,O,R,F') \mid S .  
\end{align*}
This is easily checked by graphical models \pref{fig:pomdp2}. 

\subsection{Proof of \pref{lem:lemma_value}}

From the Bellman equation, we have 
\begin{align*}
\EE[\mu(O,A)\{R+\gamma V^{\pi^e}(S')\}- V^{\pi^e}(S)\mid S]=0.  
\end{align*}
Then, from the definition of future-dependent value functions, 
\begin{align*}
\EE[\mu(O,A)\{R+\gamma \EE[g_V(F') \mid S' ] \}-\EE[g_V(F) \mid S ]\mid S]=0.  
\end{align*}
Here, we use the stationarity assumption to ensure $ \EE[g_V(F') \mid S' ]= V^{\pi^e}(S') $. 

Next, by using $F' \perp S,O,A  \mid S'$, %
we have 
\begin{align}\label{eq:lemma_key}
    \EE[\mu(O,A)\{R+\gamma g_V(F')  \}-g_V(F) \mid S]=0. 
\end{align}
More specifically, 
\begin{align*}
&\EE[\mu(O,A)\{R+\gamma g_V(F')  \}-g_V(F) \mid S]\\
&= \EE[\mu(O,A)\{R+\gamma \EE[g_V(F') \mid S',S,O,A,Z ] \}-\EE[g_V(F) \mid S ]\mid S] \tag{Law of total expectation}\\
&=\EE[\mu(O,A)\{R+\gamma \EE[g_V(F') \mid S' ] \}-\EE[g_V(F) \mid S ]\mid S] \tag{Use $F' \perp S,O,A  \mid S'$} \\
&=0. 
\end{align*}
Hence, 
\begin{align*}
&\EE[\mu(O,A)\{R+\gamma g_V(F')  \}-g_V(F) \mid  H] \\
&= \EE[\EE[\mu(O,A)\{R+\gamma g_V(F')  \}-g_V(F)\mid S,H ] \mid  H] \tag{Law of total expectation} \\
&= \EE[\EE[\mu(O,A)\{R+\gamma g_V(F')  \}-g_V(F)\mid S ] \mid  H] \tag{Use $R,O,A,F' \perp ( H)\mid S $ }\\ 
&=0.  \tag{From \pref{eq:lemma_key}}
\end{align*}Thus, $g_V $ is a learnable future-dependent value function. 

\subsection{Proof of \pref{thm:identify}}

It follows from \pref{thm:identify2}, which is an improved version of  \pref{thm:identify}.

\subsection{Proof of \pref{lem:equivalent}}

The first statement is straightforward noting the equation is equal to solving 
\begin{align*}
     x^{\top} \mathrm{Pr}_{\pi^b}( \Fb \mid {\mathbf{S}}_b) = y 
\end{align*}
for $x$ given $y$. 

The second statement is straightforward noting it is equivalent to 
\begin{align*}
    x^{\top}\mathrm{Pr}_{\pi^b}({\mathbf{S}}_b \mid \Hb)=0
\end{align*}
implies $x=0$. This is satisfied when $\rank(\mathrm{Pr}_{\pi^b}(\mathbf{S}_b,\Hb))= | \Scal_{b}|$.

\section{Proof of \pref{sec:finite}}

\subsection{Proof of \pref{thm:bellman} and \pref{thm:final_error2} }

By simple algebra, the estimator is written as 
\begin{align*}
    \hat b_V = \inf_{q\in \Qcal}\sup_{\xi  \in \Xi}\EE_{\Dcal}[(\Zcal f)^2- \{\Zcal f- \lambda \xi(H)\}^2 ] 
\end{align*}
where
\begin{align*}
    \Zcal f = \mu(A,O)\{R + \gamma q(F')\}- q(F). 
\end{align*}
Noting this form similarly appears in the proof of \citep[Theorem 3]{shi2021minimax}, the following is similarly completed by \citep[Theorem 3]{shi2021minimax}:  
\begin{align*}
    \EE[ \{\Tcal(\hat b_V)\}^2(H) ]^{1/2} = \tilde O \prns{\max(1,C_{\Qcal},C_{\Xi})\prns{1/\lambda+\lambda}\sqrt{\frac{\ln( |\Gcal| |\Xi| /\delta)}{n} }}. 
\end{align*}
using realizability and the Bellman completeness. Then, we have 
\begin{align*}
    &(J(\pi^e)-\EE_{ f \sim \nu_{ \Fcal}}[ \hat b_V ( f)])^2\\
    &= \braces{(1-\gamma)^{-1}\EE_{(\tilde s) \sim d_{\pi^e}}\left[ \EE[\mu(A,O) R\mid S=\tilde s]\right] - \EE_{ f \sim \nu_{ \Fcal}}[ \hat b_V ( f)]}^2 \\ 
    &= \braces{(1-\gamma)^{-1} \EE_{(\tilde s) \sim d_{\pi^e}}\left[ \EE\bracks{\mu(A,O)\{R+ \gamma \hat b_V (F')\} - \hat b_V (F) \mid S=\tilde s}  \right]}^2 \tag{Use \pref{lem:auxi}}  \\
    &\leq (1-\gamma)^{-2}\EE[ \{\Tcal^{\Scal}(\hat b_V)\}^2s ] \tag{Jensen's inequality} \\ 
    &\leq    (1-\gamma)^{-2} \EE[ \{\Tcal(\hat b_V)\}^2(H) ]\times  \sup_{q \in \Qcal}\frac{\EE_{ s \sim d_{\pi^e}}[(\Tcal^{\Scal} q)^2s ]}{\EE[(\Tcal q)^2(H) ]}. 
\end{align*}

\section{Proof of \pref{sec:history_based}}

In this section, most of the proof follows by slightly modifying the proof for memoryless policies. For completeness, we provide the proof of \pref{thm:history}. As we did in the proof of \pref{thm:identify}, we prove the following stronger statement. 

\begin{theorem}[Refined identification theorem ] \label{thm:refined}
Suppose (\ref{thm:identify2}a) $\Bcal_V \cap \Qcal \neq \emptyset$, (\ref{thm:identify2}b) any element in $q \in \Qcal$ that satisfies $
 \EE[\{\Tcal^S(q)\}(\bar S) \mid H] = 0 $ also satisfies $\Tcal^S(q)(\bar S)=0$. 
(\ref{thm:identify2}c) the overlap $\mu(Z,O,A)<\infty$ and any element in $q \in \Qcal$ that satisfies $
    \Tcal^S(q)(\bar S) =0$ also satisfies $\Tcal^S(q)(\bar S^{\diamond}) =0$ 
where $S^{\diamond} \sim d_{\pi^e}(s)$. 
Under the above three conditions, for any $b_V \in \Bcal_V \cap \Qcal $, we have $$
         J(\pi^e) = \EE_{\bar f \sim \nu_{\bar \Fcal}}[b_V (\bar f) ]. $$
\end{theorem}

\subsection{Proof of \pref{thm:refined} }

Note for any $b_V \in \Bcal_V \cap \Qcal$, we have 
\begin{align*}
0 &= \EE\bracks{\mu(Z,A,O) \prns{R  +  \gamma  b_V (Z',F')  } -b_V (Z,F) \mid H } \\
 &= \EE\bracks{ \EE\bracks{\mu(Z,A,O) \prns{R  +  \gamma  b_V (Z',F')  } -b_V (Z,F)\mid Z,S,H}\mid H}  \tag{Law of total expectation} \\
&=\EE\bracks{ \EE\bracks{\mu(Z,A,O) \prns{R  +  \gamma  b_V (Z',F')  } -b_V (Z,F)\mid Z,S}\mid H}. 
\end{align*}
In the last line, we use $(H \setminus Z ) \perp (A,O,F') \mid Z,S $. From (\ref{thm:identify2}b), we have
\begin{align}\label{lem:auxi2}
    \EE[b_V (Z,F)    \mid Z,S] = \EE\bracks{ \mu(Z,A,O) \prns{R  +  \gamma  b_V (Z',F')  } \mid Z,S }. 
\end{align}
Hence, $ \Tcal^{\Scal}(b_V)(S)=0. $ Then, from the overlap condition (\ref{thm:identify2}c), 
\begin{align}
 \Tcal^{\Scal}(b_V)(S^{\diamond})=0. 
\end{align}

Finally, for any $b_V \in \Bcal_V \cap \Qcal$, we have 
\begin{align*}
    &(J(\pi^e)-\EE_{\bar f \sim \nu_{\bar \Fcal}}[ b_V (\bar f)]) \\
    &= (1-\gamma)^{-1}\EE_{(\tilde s) \sim d_{\pi^e}}\left[ \EE[\mu(Z,A,O) R\mid \bar S=\tilde s]\right] - \EE_{\bar f \sim \nu_{\bar \Fcal}}[ b_V (\bar f)]  \\ 
    &= (1-\gamma)^{-1} \EE_{(\tilde s) \sim d_{\pi^e}}\left[ \EE\bracks{\mu(Z,A,O)\{R+ \gamma b_V (Z',F')\} - b_V (Z,F) \mid \bar S=\tilde s}  \right] \tag{Use \pref{lem:auxi}}  \\
    &=0. \tag{Use $\Tcal^{\Scal}(S^{\diamond})=0. $}
\end{align*}
From the first line to the second line, we use  
\begin{align*}
    J(\pi^e) &= (1-\gamma)^{-1}\int d_{\pi^e}(z,s)r(s,a)\pi^e(a\mid z,o)\mathrm{d}(z,s ) \\
    &= (1-\gamma)^{-1} \EE_{(\tilde s) \sim d_{\pi^e}}\left[ \EE[\mu(Z,A,O) R\mid \bar S=\tilde s]\right]. 
\end{align*}

\section{Proof of \pref{sec:hse} } \label{sec:proof_section_D}

\subsection{Proof of \pref{thm:identify2}}

Note for any $b_V \in \Bcal_V \cap \Qcal$, we have 
\begin{align*}
0 &= \EE\bracks{\mu(A,O) \prns{R  +  \gamma  b_V (F')  } -b_V (F) \mid H } \\
 &= \EE\bracks{ \EE\bracks{\mu(A,O) \prns{R  +  \gamma  b_V (F')  } -b_V (F)\mid S,H}\mid H}  \tag{Law of total expectation} \\
&=\EE\bracks{ \EE\bracks{\mu(A,O) \prns{R  +  \gamma  b_V (F')  } -b_V (F)\mid S}\mid H}. 
\end{align*}
In the last line, we use $(H ) \perp (A,O,F') \mid S $. From (\ref{thm:identify2}b), we have
\begin{align}\label{lem:auxi2}
    \EE[b_V (F)    \mid S] = \EE\bracks{ \mu(A,O) \prns{R  +  \gamma  b_V (F')  } \mid S }. 
\end{align}
Hence, $ \Tcal^{\Scal}(b_V)(S)=0. $ Then, from the overlap condition (\ref{thm:identify2}c), 
\begin{align}
 \Tcal^{\Scal}(b_V)(S^{\diamond})=0. 
\end{align}

Finally, for any $b_V \in \Bcal_V \cap \Qcal$, we have 
\begin{align*}
    &(J(\pi^e)-\EE_{ f \sim \nu_{ \Fcal}}[ b_V ( f)]) \\
    &= (1-\gamma)^{-1}\EE_{(\tilde s) \sim d_{\pi^e}}\left[ \EE[\mu(A,O) R\mid S=\tilde s]\right] - \EE_{ f \sim \nu_{ \Fcal}}[ b_V ( f)]  \\ 
    &= (1-\gamma)^{-1} \EE_{(\tilde s) \sim d_{\pi^e}}\left[ \EE\bracks{\mu(A,O)\{R+ \gamma b_V (F')\} - b_V (F) \mid S=\tilde s}  \right] \tag{Use \pref{lem:auxi}}  \\
    &=0. \tag{Use $\Tcal^{\Scal}(S^{\diamond})=0. $}
\end{align*}
From the first line to the second line, we use  
\begin{align*}
    J(\pi^e) &= (1-\gamma)^{-1}\int d_{\pi^e}(z,s)r(s,a)\pi^e(a\mid z,o)\mathrm{d}(z,s ) \\
    &= (1-\gamma)^{-1} \EE_{(\tilde s) \sim d_{\pi^e}}\left[ \EE[\mu(A,O) R\mid S=\tilde s]\right]. 
\end{align*}

\subsection{Proof of \pref{lem:linear_ex}}

From (LM2), there exists $w_1$ such that $V^{\pi^e}(S) = w^{\top}_1 \phi_{ \Scal}(S) $. Then, 
from (LM1), 
\begin{align*}
    \EE[w^{\top}_2 \phi_{ \Fcal}(F) \mid S] = w^{\top}_2 K_1 \phi_{ \Scal}(S)=w^{\top}_1 \phi_{ \Scal}(S).
\end{align*}
From (LM3), the above equation has a solution with respect to $w_2$. This concludes the proof.

\subsection{Proof of \pref{lem:linear_ex2}}

From (LM1), (LM4) and the statement of \pref{lem:linear_ex}, there exists $w_4 \in \RR^{d_{ \Scal}}$ such that 
\begin{align*}
    \EE[ \mu(O,A) \{R  + \gamma q(F')\} -q(F) \mid S]=w^{\top}_4  \phi_{ \Scal}(S).
\end{align*}
for any $q(\cdot) = w^{\top} \phi_{ \Fcal}(\cdot) \in \Qcal$. Letting $b_F(\cdot)=\{w^{\star}\}^{\top}\phi_{ \Fcal}(\cdot) $, this is because
\begin{align}\label{eq:whole}
    &\EE[ \mu(O,A) \{R  + \gamma q(F')\} -q(F) \mid S]  \\
    &=\EE[ \mu(O,A)\{\gamma q(F') - \gamma b_V(F') \} -q(F) + b_V(F) \mid S] \tag{Statement of \pref{lem:linear_ex}} \nonumber \\
    &=\EE[\mu(O,A)\{\gamma (w^{\top}-\{w^{\star}\}^{\top})\EE[\phi_{ \Fcal}(F') \mid S' ] \} -(w^{\top}-\{w^{\star}\}^{\top})\EE[\phi_{ \Fcal}( \Fcal)\mid S] \mid S  ] \nonumber \\   
 &=\EE[\mu(O,A)\{\gamma (w^{\top}-\{w^{\star}\}^{\top})K_1 \phi_{ \Scal}(S')  \} -(w^{\top}-\{w^{\star}\}^{\top})K_1 \phi_{ \Scal}(S)  \mid S  ] \tag{(LM1)}  \nonumber\\
 &=  w^{\top}_4  \phi_{ \Scal}(S). \tag{(LM4)} \nonumber
\end{align}
for some $ w_4$. 

Then, from (LM5), when $w^{\top}_4 \EE[\phi_{ \Scal}(S) \mid H]=0$, we have $w^{\top}_4  \phi_{ \Scal}(S)=0$. This is because first $\EE[w^{\top}_4  \phi_{ \Scal}(S) \mid H]=0$ implies $\EE[\{\EE[w^{\top}_4  \phi_{ \Scal}(S)\mid H]\}^2]=0$. %
Then, to make the ratio 
\begin{align*}
    \sup_{w \in \mathbb{R}^d} \frac{\EE[\{w^{\top}  \phi_{ \Scal}(S)\}^2 ] }{\EE[\{w^{\top}  \EE[\phi_{ \Scal}(S) \mid H]\}^2 ]}
\end{align*}
finite, we need $ \EE[\{w^{\top}_4  \phi_{ \Scal}(S)\}^2 ] =0 $. This implies $w^{\top}_4 \phi_{ \Scal}(S)=0$.

Here, when we have 
\begin{align*}
     0 &=  \EE[ \EE[ \mu(O,A) \{R  + \gamma q(F')\} -q(F) \mid S] \mid H] \\
     &= \EE[w^{\top}_4 \phi_{ \Scal}(S) \mid H], \tag{Just plug-in}
\end{align*}
we get $w^{\top}_{4} \phi_{ \Scal}(S)=0$, i.e., 
\begin{align*}
    \EE[ \mu(O,A) \{R  + \gamma q(F')\} -q(F) \mid S]=0. 
\end{align*}

\subsection{Proof of \pref{lem:linear_ex3}}

Letting
\begin{align*}
     \tilde w =\EE[\phi_{\Hcal}(H)\{ \gamma  \phi_{ \Fcal}(F')-\phi_{ \Fcal}(F)   \}]^{+} \EE[\mu(O,A)R \phi_{\Hcal}(H)], 
\end{align*}
we want to prove $ \tilde w ^{\top}\phi_{ \Fcal}(\cdot)$ is a learnable future-dependent value function. Then, by invoking \pref{thm:identify2}, the statement is proved. 

\paragraph{First step.}

Here, for $q \in \Qcal$, we have $(\Tcal q)(H)= a^{\top}\phi_{\Hcal}(H)$ for some vector $a \in \RR^{d_{\Hcal}}$. Here,  $\EE[(\Tcal q)^2(H)]=0$ is equivalent to $(\Tcal q)(H) = 0$. Besides, the condition $\EE[(\Tcal q)^2(H)]=0$ is equivalent to 
\begin{align*}
    a^{\top}\EE[\phi_{\Hcal}(H)\phi^{\top}_{\Hcal}(H)] a =0. 
\end{align*}
Thus, $\EE[(\Tcal q)(H)\phi^{\top}_{\Hcal}(H)]=  a^{\top}\EE[\phi_{\Hcal}(H)\phi^{\top}_{\Hcal}(H)]=\mathbf{0}$, where $\mathbf{0} \in \RR^{d_{\Hcal}}$ is a vector consisting of $0$, is a sufficient condition to satisfy $(\Tcal q)(H) = \mathbf{0}$. Hence, if $q(\cdot)= w^{\top}\phi_{ \Fcal}(\cdot) \in \Qcal$ satisfies 
\begin{align}\label{eq:good_formula}
    \EE[\phi_{\Hcal}(H)\{\mu(O,A)\{R+ \gamma q(F')\} - q(F)\}  ]=\mathbf{0}, 
\end{align}
$q(\cdot)$ is a learnable bridge function. Note the above equation is equal to 
\begin{align*}
     \EE[\phi_{\Hcal}(H)\{ \gamma  \mu(O,A)\phi_{ \Fcal}(F')-\phi_{ \Fcal}(F)   \}^{\top} ]w = \EE[\mu(O,A)R \phi_{\Hcal}(H)]. 
\end{align*}
Vice versa, i.e., any learnable future-dependent value function satisfies the above \pref{eq:good_formula} is similarly proved. 

\paragraph{Second step.}

Since there exists a linear learnable future-dependent value function in $\Qcal$ from \pref{lem:linear_ex}, we have a solution to 
\begin{align*}
     \EE[\phi_{\Hcal}(H)\{ \gamma  \mu(O,A)\phi_{ \Fcal}(F')-\phi_{ \Fcal}(F)   \}^{\top} ]w = \EE[\mu(O,A)R \phi_{\Hcal}(H)]
\end{align*}
with respect to $w$. We denote it by $ w$. Thus, $\tilde w$ is also a solution since 
\begin{align*}
     B \tilde w = B B^+ B  w= B  w= \EE[\mu(O,A)R \phi_{\Hcal}(H)] 
\end{align*}
where $ B= \EE[\phi_{\Hcal}(H)\{ \gamma  \phi_{ \Fcal}(F')-\phi_{ \Fcal}(F)   \}^{\top}]$.  We use  $B = BB^+B$. Note $\tilde w$ and $ w$ can be different.

\subsection{Proof of \pref{lem:lqg}}

Refer to \cite[Chapter J]{uehara2022provably}.

\section{Proof of \pref{sec:finitez_horizon2}}

\subsection{Proof of  \pref{thm:identify4}}

Take $b^{[t]}_V \in \Bcal^{[t]}_V \cap \Qcal_t$. Then, we have
\begin{align*}
   \EE[ \mu(O,A)\{R + \gamma b^{[t+1]}_V(F')\}-b^{[t]}_V(F) \mid H ]=0. 
\end{align*}
Here, this implies 
\begin{align*}
   \EE[ \EE[\mu(O,A)\{R + \gamma b^{[t+1]}_V(F')\}-b^{[t]}_V(F)\mid  S]  \mid H ]=0
\end{align*}
using $H \perp  (A,O,R,F') \mid S$. 

Then, from the we have $\Tcal^{\Scal}_t(b_V)(S)=0$, i.e., 
\begin{align*}
    \EE[ \mu(O,A)\{R + \gamma b^{[t+1]}_V(F')\}-b^{[t]}_V(F) \mid S]=0. 
\end{align*}
using the assumption (b). 
Next, from the overlap assumption (c), we have $\{\Tcal^{\Scal,t}(b_V)\}(S^{\diamond}_t)=0$ where $S^{\diamond}_t \sim d^{\pi^e}_t(\cdot)$. 

Therefore, we have
{\small 
\begin{align*}
     & J(\pi^e)- \EE_{ f \sim \nu_{ \Fcal}}[b^{[0]}_V( f)] \\
      &= \prns{\sum_{t=0}^{T-1}\EE_{ s \sim d^{\pi^e}_t}[ \gamma^t \EE[\mu(O,A) R \mid S= s]} + \prns{\sum_{t=0}^{T-1}\gamma^t\EE_{ s \sim d^{\pi^e}_{t+1} }[ \EE[b^{[t+1]}_V(F)\mid S= s]]-  \EE_{ s \sim d^{\pi^e}_{t} }[\EE[b^{[t]}_V(F) \mid S= s  ]]   } \tag{Telescoping sum} 
 \end{align*}  
 } 
from telescoping sum. Besides, we have 
\begin{align*}
  &\EE_{ s \sim d^{\pi^e}_{t+1} }[ \EE[b^{[t+1]}_V(F)\mid S= s]]\\
  & =  \EE_{ s \sim d^{\pi^e}_{t} }[\EE[\mu(O,A)\EE[b^{[t+1]}_V(F)\mid S']  \mid S= s]] \\ 
  & =  \EE_{ s \sim d^{\pi^e}_{t} }[\EE[\mu(O,A)\EE[b^{[t+1]}_V(F)\mid S',O,A]  \mid S= s]] \tag{$F \perp (O,A) \setminus S'\mid S'$ } \\
  & = \EE_{ s \sim d^{\pi^e}_{t} }[\EE[\mu(O,A)b^{[t+1]}_V(F)  \mid S= s]] \tag{Total law of expectation}. 
\end{align*}
Therefore, 
 \begin{align*}
 & J(\pi^e)- \EE_{ f \sim \nu_{ \Fcal}}[b^{[0]}_V( f)]  \\
     &= \prns{\sum_{t=0}^{T-1}\EE_{ s \sim d^{\pi^e}_t}[ \gamma^t \EE[\mu(O,A) R \mid S= s]} + \prns{\sum_{t=0}^{T-1}\gamma^t\EE_{ s \sim d^{\pi^e}_t }[ \EE[\gamma\mu(O,A)b^{[t+1]}_V(F')-b^{[t]}_V(F) \mid S= s  ]]   }\\
      &= \sum_{t=0}^{T-1}\gamma^t\EE_{ s \sim d^{\pi^e}_t }[ \EE[\mu(O,A)\{R +\gamma b^{[t+1]}_V(F') \}- b^{[t]}_V(F)\mid S=  s ]] \\ 
     &=\sum_{t=0}^{T-1}\EE_{ s \sim d^{\pi^e}_t }[ \Tcal^{\Scal}_t s]  \\
     &=0.  \tag{Recall we derive  $\{\Tcal^{\Scal,t}(b_V)\}(S^{\diamond}_t)=0$. }
\end{align*}

\subsection{Proof of Corollary \pref{cor:linear_finite} }

The proof consists of three steps. We use \pref{thm:identify4}. 

\paragraph{First step: verify existence of learnable future-dependent value functions (Show  (\ref{thm:identify4}a)). }

From (LM1), we need to find a solution to
\begin{align*}
    w^{\top}_1\EE[\phi_{ \Fcal}(F) \mid S ]= V^{\pi^e}_t(S)  
\end{align*}
with respect to a value $w_1$. Then, from (LM2f), there exists $w_2$ and $K_1$ such that
\begin{align*}
    w^{\top}_1 K_1 \phi_{ \Scal}(S)= w^{\top}_2 \phi_{ \Scal}(S). 
\end{align*}
Thus, from (LM3), the above has a solution with respect to $w_1$. 

\paragraph{Second step: verify invertibility condition (Show  (\ref{thm:identify4}b)).}

Take a function $q^{[t]}(F)$ linear in $\phi_{ \Fcal}(F)$. From (LM1), (LM2) and (LM4), there exists $w_3 \in \RR^{d_{ \Scal}}$ such that 
\begin{align*}
    \EE[\mu(O,A)\{R+\gamma q^{[t+1]}(F')- q^{[t]}(F)\} \mid S]= w^{\top}_3 \phi_{ \Scal}(S). 
\end{align*}
This is proved as in \pref{eq:whole}. Then, $  w^{\top}_3 \EE[\phi_{ \Scal}(S) \mid  H] =0 $ implies $  w^{\top}_3 \phi_{ \Scal}(S) =0 $ from (LM5). Hence,  when we have 
\begin{align*}
    0 &=\EE[ \EE[\mu(O,A)\{R+\gamma q^{[t+1]}(F')- q^{[t]}(F)\}\mid S ]\mid H] \\
    &= \EE[ w^{\top}_3 \phi_{ \Scal}(S) \mid H], 
\end{align*}
this implies $w^{\top}_3 \phi_{ \Scal}(S)=0$, i.e., 
$$\EE[\mu(O,A)\{R+\gamma q^{[t+1]}(F')- q^{[t]}(F)\}\mid S ]=0.$$ 

\paragraph{Third step: show learnable future-dependent value functions are future-dependent value functions.}

Take a learnable future-dependent value function $b_V$. Then, from the condition, we have 
\begin{align*}
    \EE[\mu(O,A)\{R + \gamma b^{[t+1]}_V(F')\}- b^{[t]}_V(F)\mid S] = 0. 
\end{align*}
We want to prove
\begin{align*}
    \EE[b^{[t]}_V(F) \mid S]= V^{\pi^e}_t(S). 
\end{align*}
We use induction. When $t=T-1$, this is clear. Next, supposing the statement is true at $t+1$, we prove it at a horizon $t$. Here, we have
\begin{align*}
    \EE[ b^{[t]}_V(F)\mid S] &= \EE[\mu(O,A)\{R + \gamma b^{[t+1]}_V(F')\} \mid S]  \\ 
    & = \EE[\mu(O,A)\{R + \gamma \EE[b^{[t+1]}_V(F') \mid S' ]\} \mid S] \\ 
    &= \EE[\mu(O,A)\{R + \gamma V^{\pi^e}_{t+1}(S')\} \mid S] =V^{\pi^e}_{t}(S).
\end{align*}
Thus, from induction, we have $\EE[b^{[t]}_V(F) \mid S]= V^{\pi^e}_t(S)$ for any $t\in [T-1]$ for any learnable future-dependent value function $b_V(\cdot)$. 

\paragraph{Fourth step: show the final formula.}

Recall we define 
\begin{align*}
    \theta_t  =\EE[\phi_{\Hcal}(H)\phi_{ \Fcal}(F)^{\top} ]^{+}\EE[\mu(O,A)\phi_{\Hcal}(H)\{R + \gamma \phi^{\top}_{ \Fcal}(F') \theta_{t+1}\} ]
\end{align*}
We want to show $\theta^{\top}_t \phi_{ \Fcal}(\cdot)$ is a learnable future-dependent value function. Here, we need to say 
\begin{align*}
    \EE[\mu(O,A)\{R + \gamma \theta^{\top}_{t+1} \phi_{ \Fcal}(F')\}- \theta^{\top}_t \phi_{ \Fcal}(F) \mid  H] = 0. 
\end{align*}
This is satisfied if we have 
\begin{align}\label{eq:main_statement}
   \EE\prns{\phi_{\Hcal}(H) \prns{\mu(O,A)\{R + \gamma \phi^{\top}_{ \Fcal}(F') \theta_{t+1} \}- \phi^{\top}_{ \Fcal}(F) \theta_t }}= 0. 
\end{align}
This is because $\EE[\mu(O,A)\{R + \gamma \theta^{\top}_{t+1} \phi_{ \Fcal}(F')\}- \theta^{\top}_t \phi_{ \Fcal}(F) \mid  H]= w^{\top}_5 \phi_{\Hcal}(H)$ for some vector $w_5$. Besides, $\EE[w^{\top}_5 \phi_{\Hcal}(H) \phi_{\Hcal}(H) ]=0$ implies $\EE[w^{\top}_5 \phi_{\Hcal}(H) \phi_{\Hcal}(H)w_5]=0$, which results in  
\begin{align*}
    w^{\top}_5 \phi_{\Hcal}(H)=0. 
\end{align*}

On top of that, since future-dependent value functions $\langle \tilde \theta_t , \phi_{ \Fcal}(F)\rangle$ exist from the first statement, we have a solution: 
\begin{align*}
   \EE\prns{\phi_{\Hcal}(H) \prns{\mu(O,A)\{R + \gamma \phi^{\top}_{ \Fcal}(F') \tilde \theta_{t+1} \}- \phi^{\top}_{ \Fcal}(F) \tilde \theta_t }}= 0. 
\end{align*}
with respect to $\tilde \theta_t$. In the following, we use this fact. 

Now, we go back to the proof of the main statement, i.e., we prove \pref{eq:main_statement}. We use the induction. This is immediately proved when $t=T-1$. Here, suppose $ \phi^{\top}_{ \Fcal}(\cdot)\theta_{t+1}$ is a learnable future-dependent value function at $t+1$. Then, we have 
\begin{align*}
    &   \EE[\phi_{\Hcal}(H)\phi_{ \Fcal}(F)^{\top} ]  \theta_t \\ 
    &=\EE[\phi_{\Hcal}(H)\phi_{ \Fcal}(F)^{\top} ]\EE[\phi_{\Hcal}(H)\phi_{ \Fcal}(F)^{\top} ]^{+}\EE[\mu(O,A)\phi_{\Hcal}(H)\{R + \gamma \phi^{\top}_{ \Fcal}(F') \theta_{t+1}\} ] \tag{Definition}\\
    & = \EE[\phi_{\Hcal}(H)\phi_{ \Fcal}(F)^{\top} ]\EE[\phi_{\Hcal}(H)\phi_{ \Fcal}(F)^{\top} ]^{+}\EE[\mu(O,A)\phi_{\Hcal}(H)\{R + \gamma V^{\pi^e}_{t+1}(S')\} ].  
\end{align*}
In the last line, we use the inductive hypothesis and the third step.  Recall we showed in the previous step $\EE[b^{[t]}_V(F) \mid S]= V^{\pi^e}_t(S)$ for any learnable future-dependent value functions $b^{[t]}_V(\cdot)$ . 
Then, we have 
\begin{align*}   
    &\EE[\phi_{\Hcal}(H)\phi_{ \Fcal}(F)^{\top} ]  \theta_t  \\
    & = \EE[\phi_{\Hcal}(H)\phi_{ \Fcal}(F)^{\top} ]\EE[\phi_{\Hcal}(H)\phi_{ \Fcal}(F)^{\top} ]^{+}\EE[\mu(O,A)\phi_{\Hcal}(H)\{R + \gamma V^{\pi^e}_{t+1}(S')\} ] \\
    & = \EE[\phi_{\Hcal}(H)\phi_{ \Fcal}(F)^{\top} ]\EE[\phi_{\Hcal}(H)\phi_{ \Fcal}(F)^{\top} ]^{+}\EE[\mu(O,A)\phi_{\Hcal}(H)\{R + \gamma \phi^{\top}_{ \Fcal}(F') \tilde \theta_{t+1}\} ] \\
    & = \EE[\phi_{\Hcal}(H)\phi_{ \Fcal}(F)^{\top} ]\EE[\phi_{\Hcal}(H)\phi_{ \Fcal}(F)^{\top} ]^{+}\EE[\phi_{\Hcal}(H)\phi_{ \Fcal}(F)^{\top} ]\tilde \theta_t  \\
    & = \EE[\phi_{\Hcal}(H)\phi_{ \Fcal}(F)^{\top} ]\tilde \theta_t \tag{Property of Moore-Penrose Inverse}\\
    &= \EE[\mu(O,A)\phi_{\Hcal}(H)\{R + \gamma \phi^{\top}_{ \Fcal}(F') \tilde \theta_{t+1}\} ] \tag{Definition of $\tilde \theta_t$}\\
   &= \EE[\mu(O,A)\phi_{\Hcal}(H)\{R + \gamma  V^{\pi^e}_{t+1}(S')\} ] \tag{We showed in the previous step}\\
     &= \EE[\mu(O,A)\phi_{\Hcal}(H)\{R + \gamma \phi^{\top}_{ \Fcal}(F') \theta_{t+1}\} ].  \tag{From the induction}
\end{align*}
Hence,  $ \phi^{\top}_{ \Fcal}(\cdot)\theta_{t}$ is  a learnable future-dependent value function at $t$.

\section{Proof of \pref{sec:modeling}}

\subsection{Proof of \pref{thm:system_identification1}} 

We take a value bridge function $ b^{[t]}_D \in \Qcal_t \cap \Bcal^{[t]}_{\Dcal}$ for any $[T]$. Then, we define 
\begin{align*}
    l^{[t]}_D(\cdot)  =  \EE[b^{[t]}_D(F) \mid  S=\cdot]. 
\end{align*}
In this section, $\EE[\cdot\,; A_{0:T-1} \sim \pi^e]$ means taking expectation when we execute a policy $\pi_e$ from $t=0$ to $T-1$. Note $\EE[\cdot\,; A_{0:T-1} \sim \pi^b]$ is just $\EE[\cdot]$. 
Here, we have
{\small 
\begin{align}
    &\Prr_{\pi^e}(o_0,a_0,\cdots,a_{T-1})- \EE_{f \sim \nu_{\Fcal}}[ b^{[0]}_D (f)] \nonumber \\
    &= \sum_{t=0}^{T-1}\EE\bracks{ \braces{\prod_{k=0}^{t-1} \rI(O_k=o_k,A_k=a_k)} \braces{\rI(O_t=o_t,A_t=a_t)l^{[t+1]}_D(S_{t+1}) -l^{[t]}_D(S_t)   }; A_{0:T-1} \sim \pi^e}\nonumber  \\ 
    & = \sum_{t=0}^{T-1}\EE[ \braces{\prod_{k=0}^{t-1} \rI(O_k=o_k,A_k=a_k)} \{\EE[\rI(O_t=o_t,A_t=a_t)l^{[t+1]}_D(S_{t+1})\mid S_t,(O_0,A_0,\cdots,A_{t-1})=(o_0,a_0,\cdots,a_{t-1}) ] \nonumber  \\ 
    & -l^{[t]}_D(S_t)   \} ; A_{0:T-1} \sim \pi^e ]\nonumber  \\ 
    & = \sum_{t=0}^{T-1}\EE\bracks{ \braces{\prod_{k=0}^{t-1} \rI(O_k=o_k,A_k=a_k)} \braces{\EE[\rI(O_t=o_t,A_t=a_t)l^{[t+1]}_D(S_{t+1})\mid S_t ] -l^{[t]}_D(S_t)   } ; A_{0:T-1} \sim \pi^e }. \nonumber 
\end{align}
}
In the first line, we use a telescoping sum trick noting
\begin{align*}
    \Prr_{\pi^e}(o_0,a_0,\cdots,a_{T-1})  = \EE\bracks{\prod_{k=0}^{T-1} \rI(O_k=o_k,A_k=a_k); A_{0:T-1} \sim \pi^e}. 
\end{align*}
From the second line to the third line, we use $S_{t+1}, O_t, A_t \perp O_1,\cdots, A_t \mid S_t$.

Here, we have 
\begin{align*}
    \EE[\rI(O_t=o_t,A_t=a_t) l^{[t+1]}_D(S_{t+1})\mid S_t; A_t\sim \pi^e] =  \tilde l^{[t]}_D(S_t) 
\end{align*}
where $\tilde l^{[t+1]}_D(\cdot)  =  \EE[\rI(O = o_t, A=a_t)\mu(O,A) l^{[t+1]}_D(S') \mid  S=\cdot; A_t \sim \pi^b]$ using importance sampling. Hence, the following holds: 
{\small 
\begin{align}
    &\Prr_{\pi^e}(o_0,a_0,\cdots,a_{T-1})- \EE_{f \sim \nu_{\Fcal}}[ b^{[0]}_D (f)] \nonumber \\
    &= \sum_{t=0}^{T-1}\EE\bracks{ \braces{\prod_{k=0}^{t-1} \rI(O_k=o_k,A_k=a_k)} \braces{ \tilde l^{[t+1]}_D(S_t)  -l^{[t]}_D(S_t)   } ; A_{0:T-1} \sim \pi^e }  \nonumber \\ 
    &= \sum_{t=0}^{T-1}\EE\bracks{ \braces{\prod_{k=0}^{t-1} \rI(O_k=o_k,A_k=a_k)(\Tcal^{\Scal}_t b_D)(S_t) } ; A_{0:T-1} \sim \pi^e  } \\ 
      &= \sum_{t=0}^{T-1}\EE\bracks{ \braces{\prod_{k=0}^{t-1} \rI(O_k=o_k,A_k=a_k)\EE[(\Tcal^{\Scal}_t b_D)(S_t) \mid (O_1,A_1,\cdots,A_K)=(o_1,a_1,\cdots,o_K)  ] } ; A_{0:T-1} \sim \pi^e  } \label{eq:sum} %
\end{align}
} 
From the second line to the third line, we use 
\begin{align*}
    &\EE[\rI(O = o_t, A=a_t)\mu(O,A) l^{[t+1]}_D(S') \mid  S ]\\
    &= \EE[\rI(O = o_t, A=a_t)\mu(O,A) \EE[b^{[t+1]}_D(F') \mid S' ] \mid  S] \tag{Definition} \\
    &= \EE[\rI(O = o_t, A=a_t)\mu(O,A) \EE[b^{[t+1]}_D(F') \mid S',O=o_t,A=a_t] \mid  S] \tag{Low of total expectation}\\
    &= \EE[\rI(O = o_t, A=a_t)\mu(O,A) b^{[t+1]}_D(F') \mid  S].  \tag{$F'\perp O,A \mid S'$ }
\end{align*}

Besides, we know for any learnable bridge function $b_D$, we have 
\begin{align*}
     \EE[\rI(O=o_t,A=a_t)\mu(O,A)b^{[t+1]}_D( F') -b^{[t]}_D( F)    \mid H]  =0. 
\end{align*}
Then, since 
\begin{align*}
     \EE[\EE[\rI(O=o_t,A=a_t)\mu(O,A)b^{[t+1]}_D( F') -b^{[t]}_D( F)\mid  S  ]   \mid H]  =0
\end{align*}
from the invertibility condition (b), we have
\begin{align*}
    \EE[\rI(O=o_t,A=a_t)\mu(O,A)b^{[t+1]}_D( F') -b^{[t]}_D( F)\mid  S  ]=0. 
\end{align*}
From the overlap (c), this implies 
\begin{align*}
     (\Tcal^{\Scal}_t b_D)( \grave S_t) = 0.
\end{align*}
Finally, from \pref{eq:sum}, we can conclude $\Prr_{\pi^e}(o_0,a_0,\cdots,a_{T-1})- \EE_{ f \sim \nu_{ \Fcal}}[ b^{[0]}_D ( f)]=0$.

\subsection{Proof of \pref{thm:system_identification2}}

This is proved as in \pref{thm:system_identification1} noting
\begin{align*}
    \Prr_{\pi^e}(o_0,a_0,\cdots,a_{T-1},\Ob_T)  = \EE\bracks{\phi_{\Ocal}(O_T) \prod_{k=0}^{T-1} \rI(O^{\diamond}_k=o_k,A^{\diamond}_k=a_k)}. 
\end{align*}

\subsection{Proof of Corollary~\ref{cor:system_identification}}

This is proved following Corollary \pref{cor:linear_system}. 

\subsection{Proof of Corollary~\ref{cor:linear_system}}

The proof consists of four steps. The proof largely follows the proof of Corollary \ref{cor:linear_finite}. 

\paragraph{First step: verify the existence of bridge functions (\ref{thm:system_identification1}a). }

We show there exists a bridge function linear in $\phi_{ \Fcal}( F)$. We need to find a solution to $w^{\top}_1 \EE[\phi_{ \Fcal}( F)\mid  S] = V^{\pi^e}_{D,[t]}( S)$ with respect to $w_1$. Then, from (LM2D), there exists $w_2$ and $K_1$ such that 
\begin{align*}
    w^{\top}_1 K_1 \phi_{ \Scal}( S)= w^{\top}_2 \phi_{ \Scal}( S).   
\end{align*}
Thus, from (LM3), the above has a solution with respect to $w_1$.

\paragraph{Second step: verify invertibility conditions (\ref{thm:system_identification1}b). }

We take $\Qcal_t$ to be a linear model in $\phi_{ \Fcal}( F)$. For $q^{[t]}_D \in \Qcal_t$, from (LM1), (LM2), (LM4), there exists $w_3 \in \RR^{d_{ \Scal}}$ such that 
\begin{align*}
    \EE[q^{[t]}_D( F)- \II(O=o_t,A=a_t) \mu(O,A)q^{[t+1]}_D( F') \mid  S ] = w^{\top}_3 \phi_{ \Scal}( S). 
\end{align*}
Then, $w^{\top}_3 \EE[\phi_{ \Scal}( S) \mid H]=0$ implies $w^{\top}_3\phi_{ \Scal}( S)=0$ from (LM5). Hence, 
\begin{align*}
    0 &= \EE[q^{[t]}_D( F)- \II(O=o_t,A=a_t) \mu(O,A)q^{[t+1]}_D( F') \mid H] \\ 
     &=\EE[\EE[q^{[t]}_D( F)- \II(O=o_t,A=a_t) \mu(O,A)q^{[t+1]}_D( F')\mid  S] \mid H] 
\end{align*}
implies $\EE[q^{[t]}_D( F)- \II(O=o_t,A=a_t) \mu(O,A)q^{[t+1]}_D( F')\mid  S]=0$. Thus, the invertibility is concluded. 

\paragraph{Third step: show learnable value bridge functions are value bridge functions. }

From the previous discussion, learnable value bridge functions in $\Qcal$ need to satisfy
\begin{align*}
    0 = \EE[b^{[t]}_D( F)- \II(O=o_t,A=a_t) \mu(O,A)b^{[t+1]}_D( F')\mid  S]. 
\end{align*}
Here, we want to prove 
\begin{align*}
    \EE[b^{[t]}_D( F) \mid  S] = V^{\pi^e}_{D,[t]}( S). 
\end{align*}
We use induction. When $t=T-1$, this is clear. Next, supposing the statement is true at $t+1$, we prove it at horizon $t$. Here, we have 
\begin{align*}
    \EE[b^{[t]}_D( F) \mid  S] &= \EE[\II(O=o_t,A=a_t) \mu(O,A)b^{[t+1]}_D( F')\mid  S] \\ 
    &= \EE[\II(O=o_t,A=a_t) \mu(O,A) \EE[b^{[t+1]}_D( F')\mid  S']\mid  S] \\
    &= \EE[\II(O=o_t,A=a_t) \mu(O,A) V^{\pi^e}_{D,[t+1]}( S') \mid  S] \\
    & = V^{\pi^e}_{D,[t]}( S). 
\end{align*}

\paragraph{Fourth step: show the final formula. }

We recursively define 
\begin{align*}
     \theta^{\top}_t = \theta^{\top}_{t+1} D_t B^+. 
\end{align*}
starting from $ \tilde w^{\top}_T B =  \EE[\phi_{\Hcal}(H) ].$ 
We want to show $\theta^{\top}_t \phi_{ \Fcal}(\cdot)$ is a learnable value bridge function. Here, we want to say 
\begin{align*}
    \EE[\theta^{\top}_t \phi_{ \Fcal}( F)- \II(O=o_t,A=a_t)\mu(O,A) \theta^{\top}_{t+1} \phi_{ \Fcal}( F') \mid H]=0. 
\end{align*}
This is satisfied if we have 
\begin{align*}
      \EE\bracks{ \prns{ \theta^{\top}_t \phi_{ \Fcal}( F)- \II(O=o_t,A=a_t)\mu(O,A) \theta^{\top}_{t+1} \phi_{ \Fcal}( F')}\phi_{ \Hcal}( H) }  =0 
\end{align*}
Here, refer to the fourth step in the proof of Corollary \ref{cor:linear_finite}. Besides, as we already show linear value bridge functions exist $\langle \tilde \theta_t, \phi_{\Fcal}(\cdot)\rangle $, we have a solution to: 
\begin{align*}
      \EE\bracks{ \prns{ \tilde  \theta^{\top}_t \phi_{ \Fcal}( F)- \II(O=o_t,A=a_t)\mu(O,A) \tilde  \theta^{\top}_{t+1} \phi_{ \Fcal}( F')}\phi_{ \Hcal}( H) }  =0 
\end{align*}
with respect to $\tilde \theta_t$. 

Hereafter, we use the induction. This is immediately proved when $t=T-1$. Here, suppose $\phi^{\top}_{ \Fcal}( F)\theta_{t+1}$ is a learnable value bridge function at $t+1$. Then, we have 
\begin{align*}
    & \EE[\phi_{\Hcal}(H) \phi_{ \Fcal}( F)^{\top}]\theta_t \\
    &=\EE[\phi_{\Hcal}(H) \phi_{ \Fcal}( F)^{\top}] \EE[\phi_{\Hcal}(H) \phi_{ \Fcal}( F)^{\top}]^{+}D^{\top}_t \theta_{t+1} \\
    &= \EE[\phi_{\Hcal}(H) \phi_{ \Fcal}( F)^{\top}] \EE[\phi_{\Hcal}(H) \phi_{ \Fcal}( F)^{\top}]^{+}\EE[\rI(O=O_t,A=a_t)\mu(O,A)\phi_{\Hcal}(H)\phi^{\top}_{ \Fcal}( F')\theta_{t+1 } ] \\
    &=  \EE[\phi_{\Hcal}(H) \phi_{ \Fcal}( F)^{\top}] \EE[\phi_{\Hcal}(H) \phi_{ \Fcal}( F)^{\top}]^{+}\EE[\rI(O=O_t,A=a_t)\mu(O,A)\phi_{\Hcal}(H)V^{\pi^e}_{D,[t+1]}( S')  ]. 
\end{align*}
In the last line, we use inductive hypothesis. We showed in the previous step $ \EE[b^{[t]}_D( F) \mid  S] = V^{\pi^e}_{D,[t]}( S)$ for any learnable value bridge function $b^{[t]}_D( F)$. Then, we have 
\begin{align*}
     & \EE[\phi_{\Hcal}(H) \phi_{ \Fcal}( F)^{\top}] \EE[\phi_{\Hcal}(H) \phi_{ \Fcal}( F)^{\top}]^{+}\EE[\rI(O=O_t,A=a_t)\mu(O,A)\phi_{\Hcal}(H)V^{\pi^e}_{D,[t+1]}( S')  ] \\
     &=\EE[\phi_{\Hcal}(H) \phi_{ \Fcal}( F)^{\top}] \EE[\phi_{\Hcal}(H) \phi_{ \Fcal}( F)^{\top}]^{+}\EE[\rI(O=O_t,A=a_t)\mu(O,A)\phi_{\Hcal}(H)\{ \phi^{\top}_{ \Fcal}( F')\tilde \theta_{t+1 }\} ] \\ 
     &= \EE[\phi_{\Hcal}(H) \phi_{ \Fcal}( F)^{\top}]\EE[\phi_{\Hcal}(H) \phi_{ \Fcal}( F)^{\top}]^+ \EE[\phi_{\Hcal}(H) \phi_{ \Fcal}( F)^{\top}]\tilde \theta_t \\
     &= \EE[\phi_{\Hcal}(H) \phi_{ \Fcal}( F)^{\top}]\tilde \theta_t  \tag{Property of Moore-Penrose Inverse}\\
     &= \EE[\rI(O=O_t,A=a_t)\mu(O,A)\phi_{\Hcal}(H)V^{\pi^e}_{D,[t+1]}( S')  ] \tag{We show in the previous step}\\
     &= \EE[\rI(O=O_t,A=a_t)\mu(O,A)\phi_{\Hcal}(H)\{ \phi^{\top}_{ \Fcal}( F') \theta_{t+1 }\}   ].  \tag{From the induction}
\end{align*}

So far, we prove $\theta^{\top}_t \phi_{ \Fcal}(\cdot)$ is a learnable value bridge function for $t$. Finally, by consdiering a time-step $t=0$, we can prove the target estimand is 
\begin{align*}
    \EE[\tilde w^{\top}_0 \phi_{ \Fcal}( F) ] = \EE[\phi_{\Hcal}(H)]^{\top}B^+ \braces{\prod^{0}_{t= T-1} D_t B^+ } C. 
\end{align*}

\section{Auxiliary Lemmas}

\begin{lemma}\label{lem:auxi}
Take $g \in [ \Scal \to \RR]$. Then, we have 
\begin{align*}
          0= (1-\gamma)^{-1} \EE_{ s \sim d_{\pi^e} } \left[ \EE[\mu(O,A)\gamma g(S') - g(S)\mid  S= s ] \right] + \EE_{ s \sim \nu_{S}}[gs]. 
\end{align*}
\end{lemma}
\begin{proof}
Let $d_{\pi^e}(z,s) = d^{\pi^e}_0(\cdot)$. Then, we have
\begin{align*}
  & \int g(z,s)  d_{\pi^e}(z,s) d(z,s) \\ 
  &=(1-\gamma) \int g(z,s) \sum_{t=0}^{\infty} \gamma^t d^{\pi^e}_t(z,s) d(z,s)\\ 
   &= \underbrace{(1-\gamma)\int g(z,s)d^{\pi^e}_0(z,s) d(z,s)}_{(a)} + \underbrace{\gamma (1-\gamma) \int g(z',s') \sum_{t=1}^{\infty} \gamma^{t-1} d^{\pi^e}_t(z',s') d(z',s')}_{(b)}.   
\end{align*}
We analyze the first term (a) and the second term (b). The first term (a) is $\EE_{ s \sim \nu_{S}}[gs].$ In the following, we analyze the second term.

Here, note 
\begin{align*}
    d^{\pi^e}_t(z',s') = \int \TT(s'\mid s,a)\OO(o \mid s)\pi^e(a \mid z,o)\delta^{\dagger}(z'=z) d^{\pi^e}_{t-1}(z,s) d(z,s). 
\end{align*}
where  $\delta^{\dagger}(z'=z)=\delta(f_{-1}(z')=f_{+1}(z))$. Here, $f_{-1}$ is a transformation removing the most recent tuple $(o,a,r)$ and $f_{+1}$ is a transformation removing the oldest tuple $(o,a,r)$. Hence,  
\begin{align*}
   & (1-\gamma)\sum_{t=1}^{\infty} \gamma^{t-1} d^{\pi^e}_t(z',s') \\
   &=  (1-\gamma) \int \sum_{t=1}^{\infty}\gamma^{t-1} \TT(s'\mid s,a)\OO(o \mid s) \pi^e(a \mid z,o)\delta^{\dagger}(z'=z) d^{\pi^e}_{t-1}(z,s)d(z,s) \\
   &= (1-\gamma)\int   \sum_{k=0}^{\infty}\gamma^{k} \TT(s'\mid s,a)\OO(o \mid s) \mu(z,o,a)\pi^b(a \mid z,o)\delta^{\dagger}(z'=z) d^{\pi^e}_{k}(z,s) d(z,s) \\
   &= \int   \TT(s'\mid s,a)\OO(o \mid s) \mu(z,o,a)\pi^b(a \mid z,o)\delta^{\dagger}(z'=z) d^{\pi^e}(z,s) d(z,s). 
\end{align*}
Therefore, the term (b) is 
\begin{align*}
    & \gamma (1-\gamma) \int g(z',s') \sum_{t=1}^{\infty} \gamma^{t-1} d^{\pi^e}_t(z',s') d(z',s') \\
    &=\gamma \int   g(z',s') \TT(s'\mid s,a)\OO(o \mid s) \mu(z,o,a)\pi^b(a \mid z,o)\delta^{\dagger}(z'=z) d^{\pi^e}(z,s) d(z,z',s,s') \\ 
       &= \gamma \EE_{ s \sim d_{\pi^e} } \left[ \EE[\mu(O,A)\gamma g(S') \mid  S= s ] \right]. 
\end{align*}

\end{proof}

\end{document}